%% file: main.tex
\RequirePackage{fix-cm}
\documentclass[12pt]{article}
\PassOptionsToPackage{table,dvipsnames}{xcolor}
\input{commands}

\usepackage[T1]{fontenc}
\usepackage[utf8]{inputenc}
\usepackage{amsmath,amssymb}
\usepackage{graphicx}
\usepackage{booktabs}
\usepackage[table]{xcolor}
\usepackage{array,tabularx,multirow}
\usepackage{capt-of}
\usepackage{wrapfig}
\usepackage{placeins}
\makeatletter
\@ifpackageloaded{natbib}{}{\usepackage[authoryear,round]{natbib}}
\makeatother
\usepackage{url}
\usepackage{hyperref}
\hypersetup{hidelinks}

\usepackage[ruled,linesnumbered]{algorithm2e}
\SetAlgoNoLine
\SetNlSty{textnormal}{}{:\ }
\SetAlgoCaptionSeparator{\ }
\SetAlCapNameFnt{\bfseries}
\SetAlCapFnt{\normalfont}
\SetKwInput{KwRequire}{Require}
\SetKwInput{KwEnsure}{Ensure}
\SetKw{Return}{return}

\definecolor{groupgray}{RGB}{243,243,243}
\definecolor{pamerrowsoft}{RGB}{247,246,252}
\definecolor{pamerrow}{RGB}{242,242,252}
\definecolor{gainred}{RGB}{190,45,45}
\definecolor{gaingreen}{RGB}{35,150,85}
\definecolor{algblue}{RGB}{25,55,220}
\definecolor{algred}{RGB}{225,35,25}
\newcommand{\goodgain}[1]{\hspace{2pt}{\scriptsize\textcolor{gaingreen}{(#1)}}}
\newcommand{\badgain}[1]{\hspace{2pt}{\scriptsize\textcolor{gainred}{(#1)}}}
\newcommand{\neutralgain}[1]{\hspace{2pt}{\scriptsize\textcolor{gray}{(#1)}}}
\newcommand{\pamer}{\textsc{PaMER}}
\newcommand{\pamerplus}{\textsc{PaMER}\textsuperscript{+}}

\newcolumntype{Y}{>{\centering\arraybackslash}X}
\title{Memory Control Signals Emerge Before Action in Long Horizon Agents}
\input{authors}

\hypersetup{colorlinks=true,linkcolor=mydarkblue,citecolor=mydarkblue,urlcolor=mydarkblue}
\usepackage{team-template}
\renewcommand{\TeamPaperID}{MEMORY CONTROL / LONG-HORIZON AGENTS}
\renewcommand{\TeamShortTitle}{Memory Control}
\begin{document}
\pagestyle{fancy}
\maketitle
\thispagestyle{first}

\begin{abstract}
Long horizon language model agents continuously accumulate interaction history, increasing computational cost while making relevant information harder to preserve and reuse. Existing context management methods mainly focus on how to compress or retrieve history, but largely leave open whether the model itself already represents the need for these memory operations before they occur. We study the hidden state immediately before each agent action and find that compression and recall needs are already encoded in the model's internal representations. These signals cannot be explained by simple context length or interaction progress, and they exhibit distinct formation patterns across model depth. We further show that most memory decision information is preserved in a compact recent context, while selectively restored historical evidence complements the long range dependencies that recent context misses. Based on these findings, we propose \textbf{Preaction Memory with Evidence Retrieval(PaMER)}, which combines state guided compression with external evidence retrieval. PaMER+ further introduces step level evidence selection to recover only the historical information required by the current task. Experiments on WorkBuddyBench, across multiple context management baselines and model backbones, show that our framework substantially reduces context consumption while maintaining competitive task performance.
\end{abstract}

\section{Introduction}
\label{sec:introduction}

Long horizon language model agents face a direct challenge: interaction
history continuously grows, while longer context increases computation,
latency, and memory cost and may also reduce the model's ability to use
relevant information effectively
\citep{liu2024lost,hsieh2024ruler,laban2026llms}.
A common strategy is to compress history when the context approaches its
window limit, or to trigger compression periodically according to token count
or interaction length. However, context length is not equivalent to
compression need. Two agent trajectories with the same number of tokens may
represent very different task states. One may be long because the agent has
just assembled critical evidence, while another may already contain repeated
searches, irrelevant tool outputs, and outdated intermediate results.
Recent work therefore begins to let agents manage context according to the
current task state rather than relying only on fixed length rules
\citep{li2026self,li2026sculptor,liu2026context,verma2026active,wang2026hymem}.

Compression introduces a complementary problem: information removed from the
active context may become important again later. When subsequent reasoning
depends on earlier evidence, the agent must recover or retrieve historical
information
\citep{hu2026sam,zhang2026comem,li2026acm,li2026sculptor,yu2026agentic,sun2026h,maharana2024evaluating}.
Long horizon memory management therefore involves two dynamic
decisions: whether the current history should be compressed, and whether
earlier information is needed again. Existing methods implement such
operations through rules, prompting, additional controllers, or dedicated
training
\citep{kang2025acon,li2026self,li2026sculptor,hu2026sam,yan2026memory,shen2026membuilder,xiong2026memory,huang2026ama}.
This leaves a more fundamental issue unresolved:
whether the language model already represents its memory need
internally before compression or historical recall is performed.

We study the hidden state immediately before the model produces its next
action at each agent decision point. Our hypothesis is that this preaction
hidden state has already integrated the task objective, recent evidence, and
historical dependencies, and can therefore reveal whether the current
history should be compressed or whether earlier evidence is needed again.
The final preaction hidden state predicts compression and recall with AUROC
values of \textbf{0.831} and \textbf{0.765}, respectively, substantially
outperforming observable controls based on context length, turn count, and
tool type. More importantly, we further investigate \textbf{how this
information forms inside the model}. Layerwise analysis shows that
compression information generally becomes stronger with model depth, whereas
recall information emerges earlier and reaches its highest measured
performance at an intermediate layer. The representation at the last input
token also consistently outperforms mean pooling. These results suggest that
the model contains more than a simple signal that the context has become
long. Instead, memory needs are reflected in internal representations that
depend on the current task state and develop differently during model
computation.

We next study how much history is required to preserve this decision
information. Retaining only the task prefix and the two most recent complete
interaction blocks uses approximately 26\% of the Full Context
input, yet preserves most of the predictive performance for both compression
and recall. This indicates that the recent working state already contains
most of the dynamic information required for memory decisions. Recent context
alone, however, is not sufficient. When long range dependencies become
relevant again, restoring a small amount of selected raw historical evidence
recovers additional predictive information and outperforms the hidden state
memory alternatives that we evaluate. This leads to a natural design
principle: recent context describes the current working state, while
external history restores the evidence that is currently missing.

Based on these findings, we propose \textbf{PaMER}, Preaction Memory with
Evidence Retrieval. PaMER uses the preaction representation extracted by a shared frozen
Qwen3.5-9B model to decide when to compress, removes older history from the active context, and
stores it in retrievable external memory for later task-specific evidence reuse. PaMER+ further introduces step
level evidence selection so that only the raw historical steps required by
the current task are restored.

Our contributions are threefold.
\textbf{First}, we systematically study whether compression and recall needs
are already present in the model's internal representations before the
corresponding memory operations occur. We show that preaction hidden states
contain memory decision information that cannot be explained by simple
context length or interaction progress, and we further characterize how these
signals form across model depth.
\textbf{Second}, we investigate which parts of history support these
decisions. Approximately 26\% of Full Context is sufficient to
preserve most predictive performance, while selective historical evidence
recovers long range dependencies that are missing from recent context.
\textbf{Third}, we translate these findings into PaMER and PaMER+, and
evaluate them on WorkBuddyBench~\citep{team2026tencent} against multiple context management
baselines and across several model backbones. The resulting system
substantially reduces context consumption while maintaining competitive task
performance.

\section{Related Work}
\label{sec:related-work}

\subsection{Context Compression for Long Horizon Agents}

Long context processing has motivated a broad class of methods that reduce
the amount of history presented to a language model while preserving
task relevant information. LLMLingua 2 learns token importance through
distilled supervision, while LongLLMLingua extends prompt compression to
long context settings
\citep{pan2024llmlingua,jiang2024longllmlingua}.
As language models are increasingly used as agents, recent work has moved
from static prompt compression to the management of continuously growing
interaction histories. ACON~\citep{kang2025acon} optimizes compression for long horizon agent
trajectories, SWE Pruner~\citep{wang2026swe} adaptively removes redundant
context in coding agents, and PACE~\citep{wei2026pace} selects historical
information according to its predicted relevance to future actions.

Recent studies further suggest that compression timing should depend on the
current task state rather than only on a fixed context threshold.
Self Compact allows an agent to decide when to compact its trajectory,
Self GC continuously reorganizes context during interaction, and Sculptor
provides explicit operations for summarizing, hiding, restoring, and
searching information
\citep{li2026self,hao2026self,li2026sculptor}.
State Proprioception similarly exposes runtime information to support
adaptive context management
\citep{xu2026llm}.
These methods show that context length alone does not fully characterize
when memory intervention is useful. Our work studies an earlier question in
this process. Rather than first providing an explicit context state or
training a new memory policy, we ask whether the model's own hidden
representation already contains information about the need to compress.

\subsection{Memory Retrieval and Historical Evidence}

Removing information from active context creates a complementary problem:
historical evidence may become useful again later \citep{ai2026cognitive,zhu2026hela,banerjee2026apex,cao2026higmem,ge2025tremu}. MemGPT organizes
information between active and external memory, SAM adapts memory access to
the current reasoning state, CoMem introduces a separate long context
component for context management, and ACM studies agentic context management
for long horizon tasks
\citep{packer2023memgpt,hu2026sam,zhang2026comem,li2026acm}.
Learning What Not to Forget further studies lightweight mechanisms for
deciding which parts of long agent histories should remain available
\citep{jahan2026learning,zhang2026lightweight,sun2026preference}.

Retrieval itself has also been studied as a dynamic decision. FLARE performs
retrieval according to signals observed during generation, while Self RAG
learns when to retrieve and how to use retrieved evidence through explicit
reflection
\citep{jiang2023active,asai2024self,li2026ocr,hu2026querylink}.
Most closely related to our analysis, Probing RAG uses internal
representations to determine whether additional document retrieval is
necessary
\citep{baek2025probing}.
Our setting differs in both the target and the interaction structure.
Probing RAG studies selective retrieval of external documents for question
answering, whereas we study two memory needs inside long agent trajectories:
whether accumulated interaction should be compressed and whether evidence
that previously left the active context is needed again. We further examine
how these two signals develop across model depth and how much of the
trajectory history is required to preserve them.

\subsection{Probing Internal Representations}

Our analysis is also related to work that uses probes to study information
encoded in neural representations. Linear probes provide a simple way to
measure whether a target property can be decoded from intermediate model
states
\citep{alain2016understanding}.
At the same time, prior work emphasizes that probe accuracy should be
interpreted carefully and does not by itself identify the causal mechanism
used by the underlying model
\citep{hewitt2019designing,belinkov2022probing}.

Related studies have shown that hidden representations can expose information
that is not directly available from model outputs. Such signals include
latent knowledge~\citep{burns2022discovering},
confidence about whether an answer is known~\citep{kadavath2022language,mahaut2024factual},
and information related to truthfulness~\citep{azaria2023internal}.
We use probing for a more specific systems question: whether the model state
immediately before an agent action already contains information about an
upcoming memory need. Accordingly, our claims are predictive rather than
causal. We compare hidden states with observable metadata and surface text
controls, analyze how compression and recall signals evolve across layers,
and study how these signals change when most historical context is removed.
These findings then provide the empirical basis for the design of our online
memory controller for long horizon agent interactions.

\section{Understanding Memory Needs in Long Horizon Agents}
\label{sec:setup}
\label{sec:memory-probing}
\label{sec:memory_decisions}

We study whether compression and recall needs are represented internally
before the model produces its next action. Our analysis defines these memory
decisions from long agent trajectories, extracts the corresponding preaction
hidden states, and measures how well the decisions can be read from different
model layers across successive stages of computation.

We first compare hidden states with observable controls, then examine
how the signal changes across layers. Finally, we vary the available
history to test how much decision information remains under a smaller
context budget and after historical evidence is restored.

\subsection{Memory Decision Data and Probing}
\label{sec:memory-states}
\label{sec:decision-probes}
\label{sec:layerwise-probing}

We collect 72,912 decision points from 2,521 long agent trajectories
covering 600 questions. The source questions do not overlap with
WorkBuddyBench and come from different task domains. An LLM annotator
identifies states where compression or recall is needed. The annotation
procedure and the use of multiple labeled points within one trajectory
are described in Appendix~\ref{app:annotation}.

At decision step $t$, let $o_t$ denote the context available before the
next action. For trajectory $\tau$, the annotator identifies compression
and recall decision sets $\mathcal A^{\mathrm c}(\tau)$ and
$\mathcal A^{\mathrm r}(\tau)$. We define
\begin{equation}
y_t^m
=
\mathbf{1}\!\left[t\in\mathcal A^m(\tau)\right],
\qquad
m\in\{\mathrm c,\mathrm r\},
\label{eq:memory-labels}
\end{equation}
where $\mathrm c$ denotes compression and $\mathrm r$ denotes recall.
Compression indicates that older context should be condensed while
preserving task relevant information. Recall indicates that the current
reasoning state needs earlier evidence that is no longer available in
the active context. Recall is evaluated only when valid external memory
exists. A trajectory may contain multiple positive decision points because
the same information need can arise again during a long interaction. This formulation allows memory needs to be evaluated at the decision level
while preserving their dependence on the surrounding trajectory context.

For each decision point, we extract the model state immediately before
the next action. Let $\phi^{(L)}(o_t)$ denote the final hidden sequence
after the output RMSNorm. Our main representation is
\begin{equation}
\mathbf h_t
=
\phi^{(L)}(o_t)[-1]
\in\mathbb R^d,
\qquad d=4096.
\label{eq:preaction-state}
\end{equation}
The input uses the same conversation format and tool descriptions as
the corresponding agent state. Neither the annotation nor the future
action is included when extracting $\mathbf h_t$.

We train an independent readout for each memory decision. The linear
probe predicts
\begin{equation}
p_t^m
=
\sigma\!\left(
(\mathbf w^m)^\top\mathbf h_t+b^m
\right),
\qquad
m\in\{\mathrm c,\mathrm r\}.
\label{eq:linear-probe}
\end{equation}
We additionally evaluate a one hidden layer MLP to test whether nonlinear
readout provides additional predictive power. The language model itself
is not updated during probe training. This separation is intended to test what information is already accessible
from the pretrained representation without adapting the underlying model.

Because both targets are imbalanced, we optimize class weighted binary
cross entropy,
\begin{equation}
\mathcal L_m
=
-\frac{1}{N_m}
\sum_{i=1}^{N_m}
\left[
\rho_m y_i^m\log p_i^m
+
(1-y_i^m)\log(1-p_i^m)
\right],
\label{eq:weighted-bce}
\end{equation}
where $\rho_m=N_m^-/N_m^+$ is computed from the training split only.
All states from the same question remain in the same split. We use
70/15/15 training, validation, and test splits over five seeds.
Decision thresholds are selected by validation F1 and then fixed for
testing. AUROC is the primary metric. AUPRC, thresholded metrics, and
the complete MLP definition are reported in
Appendix~\ref{app:probe-training}.

This evaluation protocol treats each question as the grouping unit rather
than individual decision points, reducing leakage between closely related
states from the same underlying task. It also ensures that probe performance
reflects generalization across questions rather than memorization of trajectory-specific
patterns. This setting more closely reflects deployment, where the controller must generalize to unseen tasks.
It also provides a stricter test of whether the learned signal transfers beyond the trajectories used for probe fitting.

To characterize how memory decision information develops across model
depth, we train an independent linear probe at each evaluated layer.
At layer $\ell$, we compare the representation at the final input
position with the mean representation over all input tokens:
\begin{equation}
\mathbf h_{t,\mathrm{last}}^{(\ell)}
=
\phi^{(\ell)}(o_t)[-1],
\qquad
\mathbf h_{t,\mathrm{mean}}^{(\ell)}
=
\frac{1}{n_t}
\sum_{j=1}^{n_t}
\phi^{(\ell)}(o_t)[j].
\label{eq:layerwise-states}
\end{equation}

For each layer, target, and pooling choice, we also measure the difference
between the mean positive and negative probe logits,
\begin{equation}
\Delta_{\ell,\alpha}^{m}
=
\mathbb E
\!\left[
z_{t,\ell,\alpha}^{m}
\mid y_t^m=1
\right]
-
\mathbb E
\!\left[
z_{t,\ell,\alpha}^{m}
\mid y_t^m=0
\right],
\label{eq:layerwise-gap}
\end{equation}
where $\alpha\in\{\mathrm{last},\mathrm{mean}\}$.
AUROC remains the primary metric because the magnitude of the logit gap
also depends on probe scale.

\subsection{Memory Signals beyond Observable Context}
\label{sec:hidden_baselines}
\label{sec:latent}

\begingroup
\setlength{\intextsep}{0pt}
\setlength{\columnsep}{10pt}
\setlength{\emergencystretch}{1em}

\begin{wraptable}{r}{0.31\textwidth}
\centering
\normalsize
\setlength{\abovecaptionskip}{0pt}
\setlength{\belowcaptionskip}{4pt}
\caption{Memory decision signals.
Mean AUROC.}
\label{tab:hidden_controls}

\setlength{\tabcolsep}{3.5pt}
\renewcommand{\arraystretch}{1.06}
\resizebox{\linewidth}{!}{%
\begin{tabular}{l|cc}
\toprule
Signal & Comp. $\uparrow$ & Recall $\uparrow$ \\
\midrule
Length / Turn & 0.751 & 0.552 \\
Metadata & 0.758 & 0.549 \\
\midrule
\rowcolor{pamerrowsoft}
Hidden (Linear) & \textbf{0.831} & \textbf{0.765} \\
\rowcolor{pamerrowsoft}
Hidden (MLP) & 0.827 & 0.763 \\
\bottomrule
\end{tabular}%
}
\end{wraptable}

A simple explanation for memory decisions is that the agent compresses
because the context is long, or recalls because the interaction has
lasted many turns. Table~\ref{tab:hidden_controls} shows that these
observable signals are informative but insufficient. Metadata reaches
$0.758$ AUROC for compression and $0.549$ for recall, whereas the final
hidden state reaches $0.831$ and $0.765$, respectively. A nonlinear MLP
provides almost no additional gain, suggesting that much of this
information is already accessible through a linear readout.

These results show that preaction hidden states contain predictive
information beyond the tested metadata. They do not imply that the
probes identify the model's causal mechanism. Additional text controls,
uncertainty estimates, and robustness results are reported in
Appendix~\ref{app:probing-details}.

\par
\endgroup

\subsection{How Memory Signals Form across Depth}
\label{sec:layerwise_results}

Figure~\ref{fig:layerwise} shows that compression prediction becomes
stronger with model depth. The earliest measured representation provides
only a weak signal, with AUROC around $0.57$, whereas the final normalized
state reaches $0.831$. Recall follows a different pattern and reaches its
highest measured mean AUROC of $0.780$ at block 16, compared with $0.765$
at the final state. The two memory needs therefore develop differently
across the network.

\begin{figure*}[t]
\centering
\begin{minipage}[t]{0.49\textwidth}
  \centering
  \includegraphics[width=\linewidth]{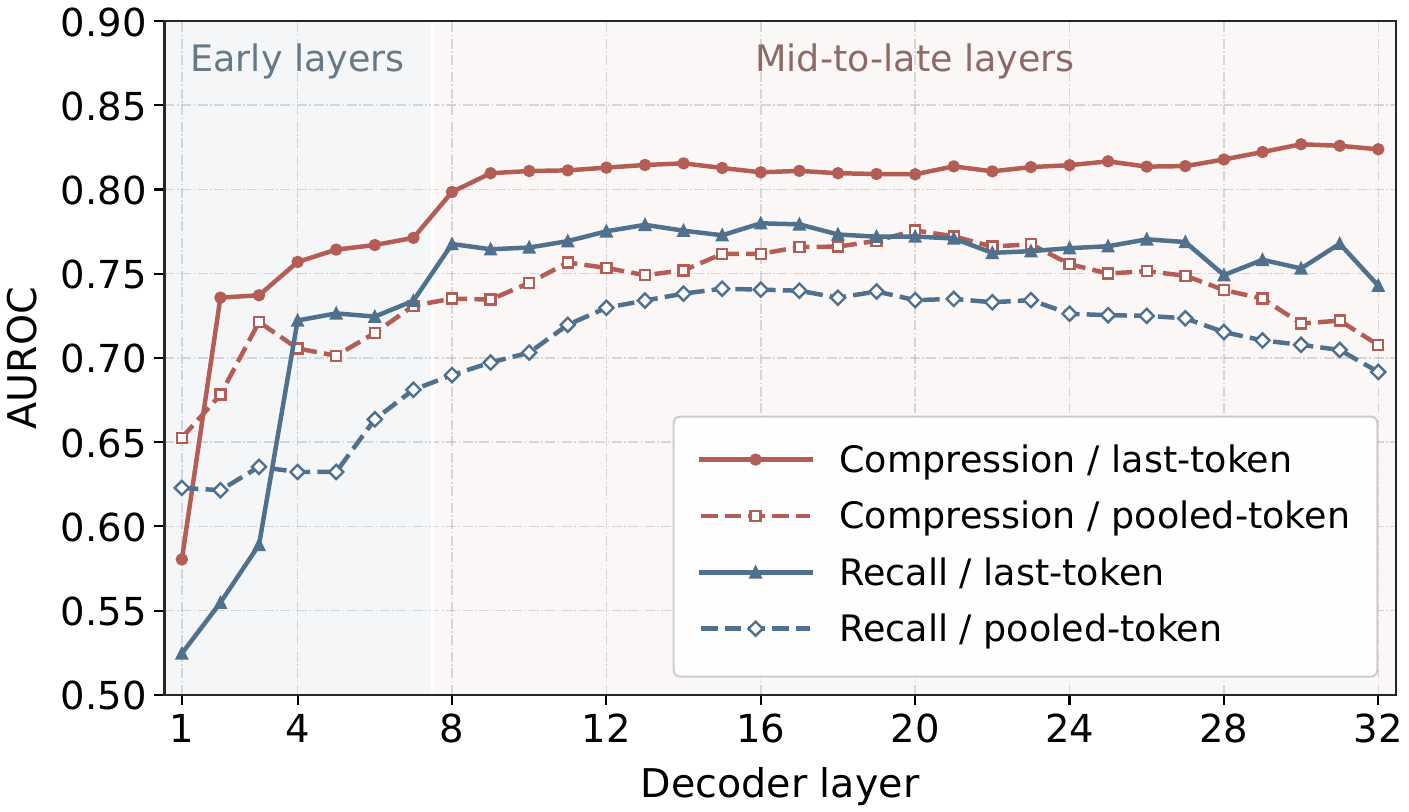}
\end{minipage}\hfill
\begin{minipage}[t]{0.49\textwidth}
  \centering
  \includegraphics[width=\linewidth]{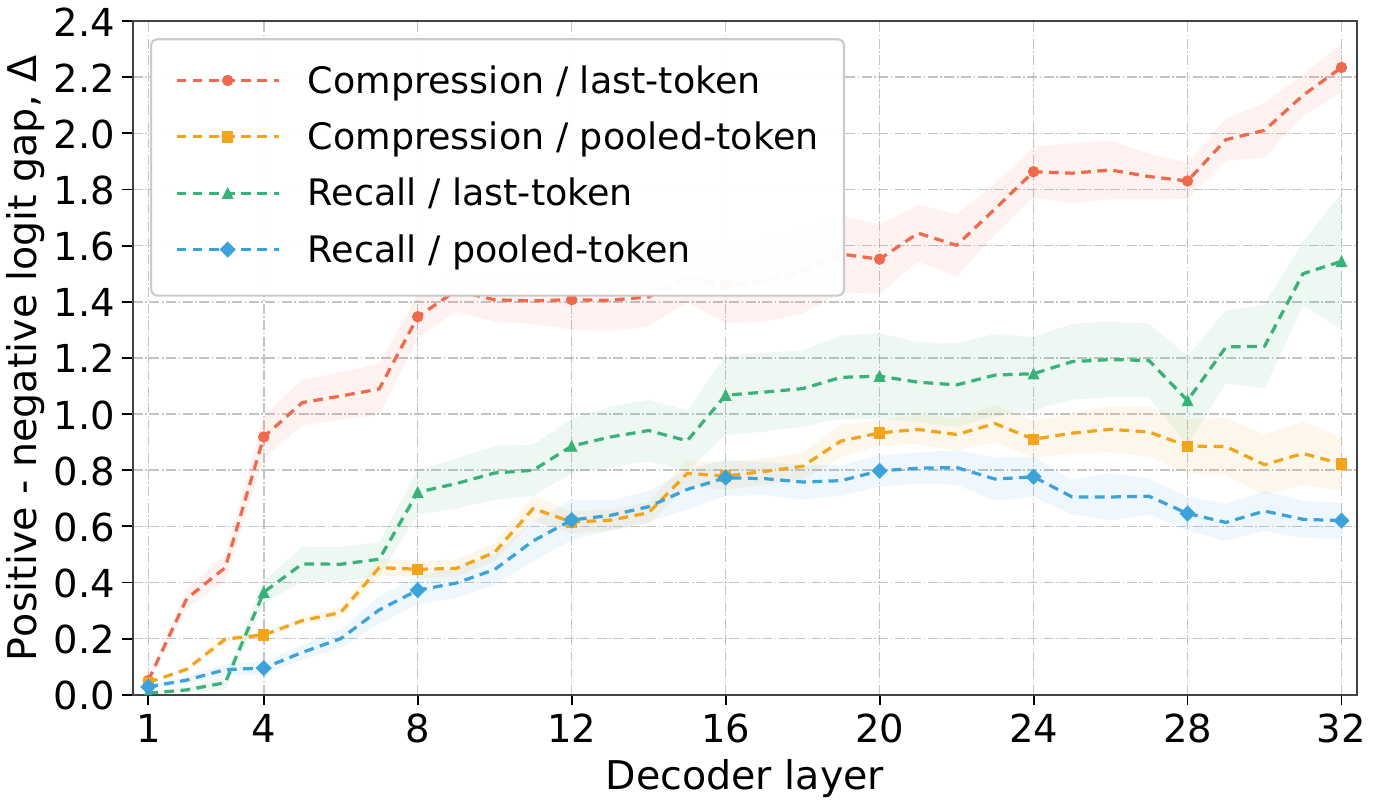}
\end{minipage}
\caption{\textbf{Memory-decision signals across layers.}
(a) AUROC of independently trained probes using the final input token
or mean-pooled states. (b) Mean positive-minus-negative probe logit.
Results are aggregated over five question-grouped splits.}
\label{fig:layerwise}
\end{figure*}

The final input position is consistently more informative than mean
pooling. At the final state, last-token and mean-pooled AUROCs are
$0.831$ and $0.712$ for compression, and $0.765$ and $0.696$ for recall.
This suggests that information relevant to the next memory action is
more accessible near the position used to produce that action than in
an average over the full input.

Although recall peaks at block 16, we use the final state for the online
controller. Compression, which is the learned online decision in
\pamer{}, improves from $0.810$ at block 16 to $0.831$ at the final
state. The recall advantage at block 16 is $0.015$, while the reported
variation across splits is larger than this difference. More importantly,
the final preaction state provides a consistent interface across models
with different depths, whereas a fixed intermediate block does not have
the same meaning across architectures. We therefore use the final state
as the controller readout for stable cross-model deployment behavior without claiming that it is the best layer for
every memory-related target. Exact layer-level results are reported in
Appendix Table~\ref{tab:layerwise-metrics}.

\subsection{What Historical Information Is Needed}
\label{sec:compact_context}
\label{sec:evidence}

We next reduce the amount of trajectory history available to the model
while keeping the prediction targets unchanged. Let $P_t$ denote the
system and task prefix and let $\mathcal R_t$ contain the two most recent
complete interaction blocks. The Recent2 condition uses
\begin{equation}
C_t^{\mathrm{recent}}
=
P_t
\oplus
\mathcal R_t.
\label{eq:recent-context}
\end{equation}

\begin{figure*}[t]
\centering
\includegraphics[width=0.96\textwidth]{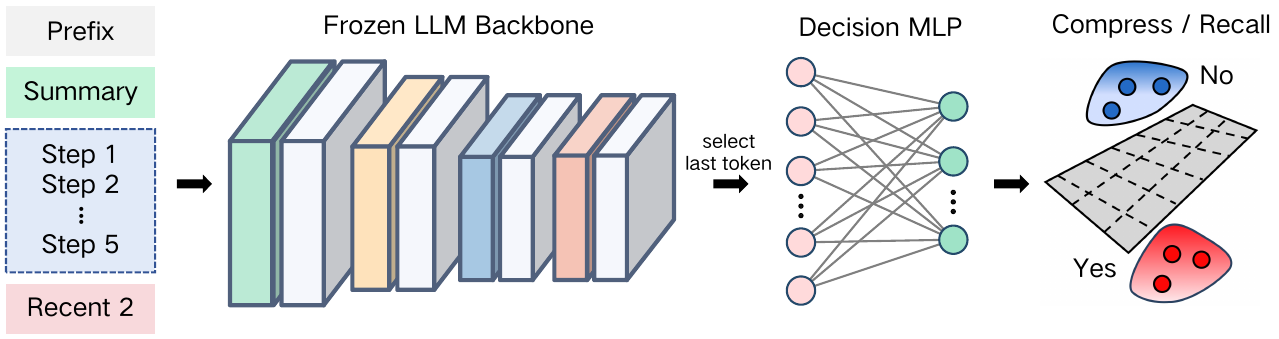}
\caption{\textbf{Reading memory needs before an action.}
The diagram illustrates compression and recall as prediction targets.
In the deployed controller, the compression heads read the final state
from the task prefix and Recent2. The recall probe is used for analysis,
not as an online retrieval gate.}
\label{fig:pamer_overview}
\vspace{-15pt}
\end{figure*}

This construction preserves the task objective together with the current
working state while removing most earlier interaction. Recent2 uses only
$26.0\%$ of the Full Context token budget, yet retains compression and
recall AUROCs of $0.776$ and $0.761$, compared with $0.829$ and $0.780$
under Full Context. Decision distillation keeps the same context budget.
It changes compression AUROC only from $0.776$ to $0.781$, but increases
compression F1 from $0.541$ to $0.598$, suggesting that Recent2 already
contains much of the required information while the Full Context teacher
mainly improves the thresholded decision.

The remaining gap shows that recency alone is not sufficient. Some
dependencies needed for a memory decision lie outside the recent window.
If $\mathcal E_t$ denotes recovered historical evidence, the reconstructed
context is
\begingroup
\setlength{\abovedisplayskip}{3pt}
\setlength{\abovedisplayshortskip}{3pt}
\setlength{\belowdisplayskip}{3pt}
\setlength{\belowdisplayshortskip}{3pt}
\begin{equation}
C_t^{\mathrm{recover}}
=
P_t
\oplus
\mathcal E_t
\oplus
\mathcal R_t.
\label{eq:recovered-context}
\end{equation}
\endgroup

\begingroup
\setlength{\intextsep}{0pt}
\setlength{\columnsep}{10pt}
\setlength{\emergencystretch}{1em}

\begin{wraptable}{r}{0.45\textwidth}
\centering
\setlength{\abovecaptionskip}{0pt}
\setlength{\belowcaptionskip}{3pt}

\scriptsize
\caption{Compact context and historical evidence.
Tokens are relative to Full Context.}
\label{tab:compact_main}

\setlength{\tabcolsep}{3.6pt}
\renewcommand{\arraystretch}{1.05}
\resizebox{\linewidth}{!}{%
\begin{tabular}{l|ccc}
\toprule
History
& Tokens $\downarrow$
& Comp. $\uparrow$
& Recall $\uparrow$ \\
\midrule

Full Context
& 100.0
& 0.829
& \textbf{0.780} \\

\midrule
\multicolumn{4}{l}{\textit{Compact context}} \\[-1pt]

Recent2
& 26.0
& 0.776
& 0.761 \\

\rowcolor{pamerrowsoft}
\quad + Decision KD
& 26.0
& 0.781
& 0.753 \\

\midrule
\multicolumn{4}{l}{\textit{Recovered history}} \\[-1pt]

Hidden Bank
& 16.8
& 0.769
& 0.730 \\

Selective Raw
& 66.2
& \textbf{0.832}
& 0.769 \\

Qwen Retrieval
& 53.9
& 0.804
& 0.764 \\

\bottomrule
\end{tabular}%
}
\vspace{-2pt}
\end{wraptable}

The task and system prefix is retained in every condition, so all
comparisons preserve the goal against which relevance is judged.
The prefix-only Current condition is reported in
Appendix~\ref{app:compact-results}. These comparisons do not isolate
the causal contribution of the task prefix.

Table~\ref{tab:compact_main} shows that a short recent history preserves
most of the recall signal and a large fraction of the compression signal.
However, the form of older information also matters. Hidden Bank Recent5
uses the smallest token budget in the table but loses more decision
information. Restoring selected raw evidence is more effective in this
setting: Selective Raw Top5 reaches $0.832$ compression AUROC and $0.769$
recall AUROC. Qwen Top3 provides a practical embedding-based alternative,
reaching $0.804$ and $0.764$ with $53.9\%$ of the Full Context token
budget, with a moderate token budget.

Taken together, the results suggest three complementary sources of
information. The task prefix establishes what is relevant, recent
interaction describes the current working state, and selected historical
evidence restores dependencies that have left the recent window. We do
not interpret this as a causal decomposition. Instead, it provides the
design pattern used by \pamer{} in the next section. Full context sweeps
and retrieval controls are reported in
Appendix~\ref{app:compact-results}.

\par
\endgroup

\section{\pamer: Preaction Memory with Evidence Retrieval}
\label{sec:method}

\pamer{} uses the current model state to decide when to compress history
and an external memory to retain information for later access.
\pamerplus{} adds selection at the level of individual historical steps.
The design follows the preceding analysis: a short recent context can
support compression decisions, while selected older evidence remains
available when information has left the active history.

\subsection{State Guided Compression and Retrieval}
\label{sec:pamer_overview}
\label{sec:pamer_controller}

\pamer{} retains the system and task prefix $P_t$ and the two most recent
complete interaction blocks, $\mathcal R_t$. A block contains an
assistant message and its subsequent tool results and other messages,
up to the next assistant message. Earlier complete blocks form the
compression candidate. Operating on complete blocks preserves the
assistant and tool-call structure.


For a candidate of at least 1,024 tokens, five independently trained
compression heads read the final state from $P_t\oplus\mathcal R_t$.
Compression is triggered when at least three heads exceed their
validation-selected thresholds. The heads are trained on Recent2 using
both the compression annotation and the corresponding Full Context
teacher prediction:
\begin{equation}
\mathcal L_s
=
\operatorname{WBCE}(z_{t,s},y_t^{\mathrm c})
+
0.5\operatorname{BCEWithLogits}
(z_{t,s},p_{t,s}^{\mathrm{full}}).
\label{eq:pamer_kd}
\end{equation}
Here $s$ indexes the independently trained head. The Full Context teacher
is used during training, not during online decisions.

When compression is triggered, we summarize the older blocks.
Each summary is encoded with Qwen3-Embedding-8B and stored with links
to the covered history. The raw blocks are removed only after storage
succeeds. The current request then contains
\begin{equation}
C'_t
=
P_t
\oplus
\mathcal R_t.
\label{eq:pamer_compressed_context}
\end{equation}
The new summary is not inserted into this same request, and no recall
is performed on it. This prevents immediate retrieval from undoing the
context reduction.

On a later request with available memory, the task and Recent2 form an
embedding query. The controller retrieves the Top3 summaries by cosine
similarity and inserts them between the prefix and recent blocks.
The current implementation does not use the learned recall probe as a
gate. Thus, recall probing and online retrieval remain distinct parts of
the study. If a memory operation fails, the controller keeps the original
history rather than applying an incomplete update. Implementation
details are in Appendix~\ref{app:controller-contract}.
\begin{figure*}[t]
\centering
\begin{minipage}[t]{0.55\textwidth}
  \vspace{0pt}
  \centering
  \includegraphics[width=\linewidth,keepaspectratio]{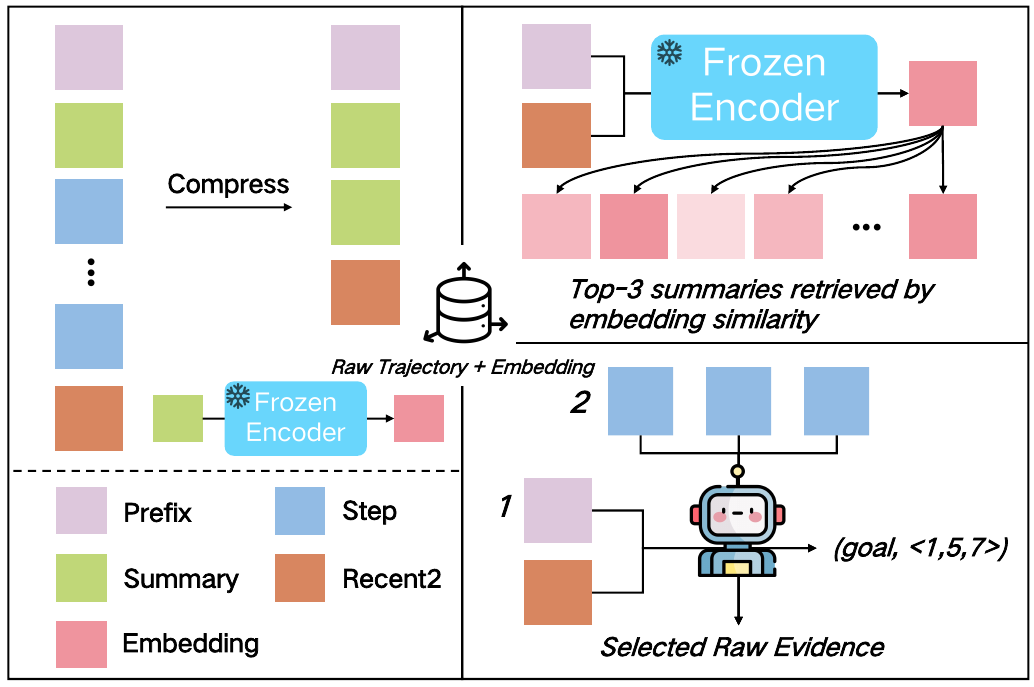}
\end{minipage}\hfill
\begin{minipage}[t]{0.43\textwidth}
  \vspace{0pt}
  \centering
  \includegraphics[width=\linewidth,keepaspectratio]{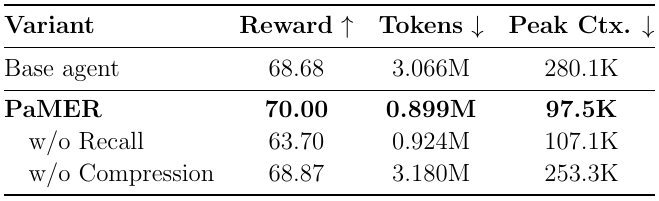}
  \vspace{0.45em}
  \includegraphics[width=\linewidth,keepaspectratio]{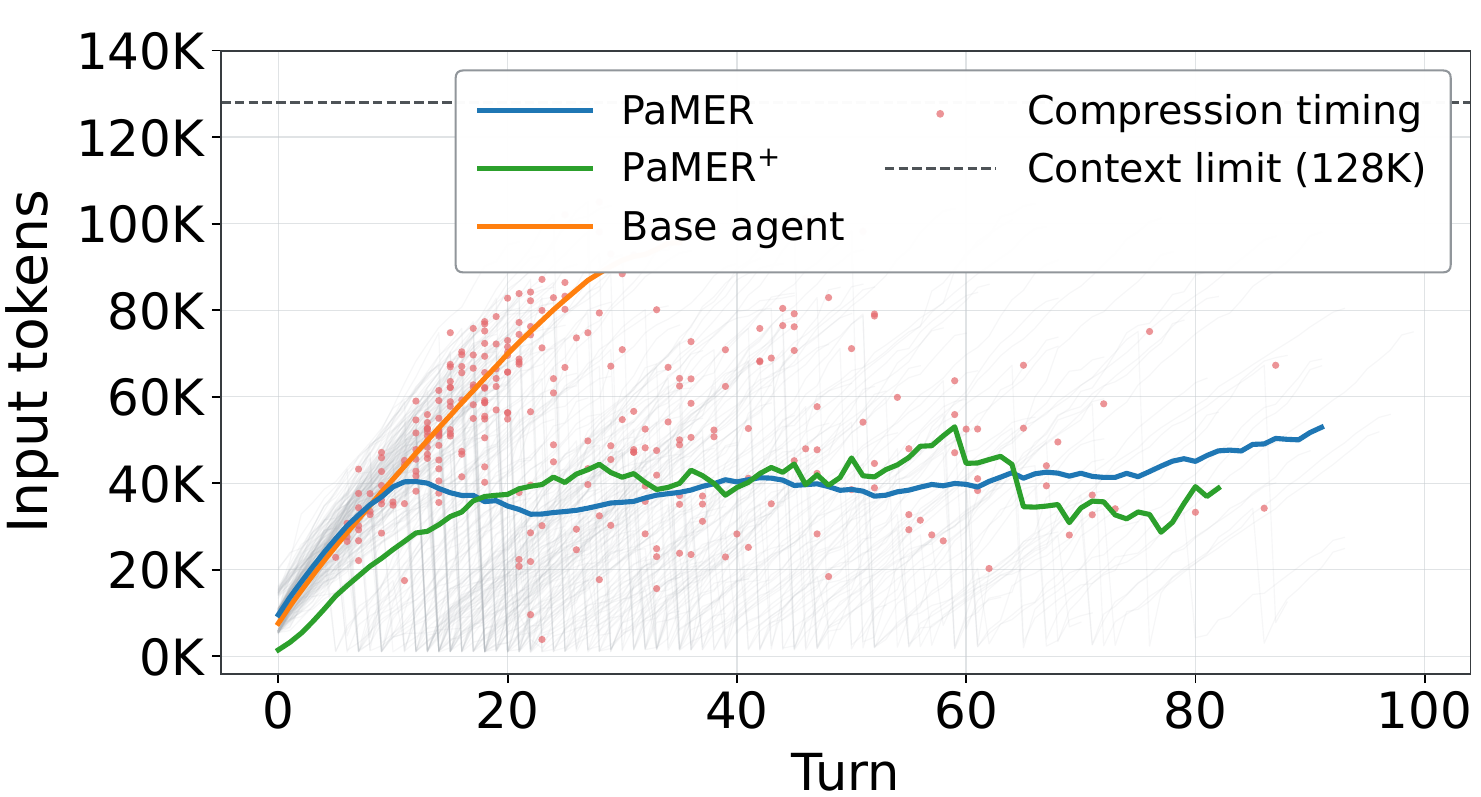}
\end{minipage}
\caption{\textbf{Online behavior and component comparisons.}
(a) Compression and evidence recall. (b) Component results.
(c) Input context across interaction turns. The horizontal line marks
128K as an analysis reference, not a hard limit enforced by the controller.}
\label{fig:system_analysis}
\vspace{-10pt}
\end{figure*}

\subsection{\pamerplus: Step Level Evidence Selection}
\label{sec:pamer_provenance}

A summary can cover many steps even when the current request needs only
a small part of them. \pamerplus{} records \texttt{step\_details} for
each summary, linking its contents to the original historical steps.
A selected summary identifies candidate memory, but does not authorize
restoring every step that it covers.
The selector specifies \texttt{raw\_steps\_needed}, and only those
complete historical blocks enter the recall path. The summary therefore
helps locate the relevant memory, while selected raw steps provide the
supporting evidence for the current reasoning state. The selected steps must belong to the chosen memory,
and invalid selections are rejected. Detailed validation, cache reuse,
and fallback rules are given in
Appendix~\ref{app:pamerplus-algorithm}.

\section{Experiments}
\label{sec:experiments}

We evaluate task performance and token consumption on the full
WorkBuddyBench benchmark, compare context-management methods on
Pilot40, and examine results across several model backbones.
We then inspect context growth and the separate roles of compression
and recall.

\subsection{Experimental Setup}
\label{sec:system_setup}

WorkBuddyBench contains 260 tasks: 80 Code, 50 Office, 60 Security,
and 70 Web. Pilot40 contains 10 fixed tasks from each domain. Methods
within a comparison use the same task set. The memory-decision corpus
is collected separately from these evaluation tasks.

We report mean task reward multiplied by 100. Overall scores average
over tasks, rather than giving equal weight to domains of different
sizes. Tokens denote mean total input and output tokens per task,
including the main agent and auxiliary memory calls. Cache tokens are
not added separately because they are already included in input usage.
Detailed evaluation protocols, token accounting, and system results are provided in
Appendix~\ref{app:system-details}.

\begin{table*}[h]
\vspace{-12pt}
\centering
\caption{
\textbf{Full benchmark system performance and token usage.}
Scores are multiplied by $100$.
Green and red indicate favorable and unfavorable changes from the base system.
}
\label{tab:domain_results}

\small
\setlength{\tabcolsep}{7pt}
\renewcommand{\arraystretch}{1.10}

\resizebox{\textwidth}{!}{%
\begin{tabular}{l|ccccc|c}
\toprule

Model & Code & Office & Sec. & Web & Avg. & Tokens $\downarrow$ \\

\midrule

DeepSeek-V4-Flash & \textbf{77.0} & \textbf{81.8} & 47.8 & 72.1 & 69.9 & 2.69M \\

\rowcolor{pamerrowsoft}
\quad + \pamer{} & 75.3 \badgain{-1.7} & 79.8 \badgain{-2.0} & \textbf{52.3} \goodgain{+4.5} & 70.0 \badgain{-2.1} & 69.4 \badgain{-0.4} & \textbf{0.72M} \goodgain{-73.2\%} \\

\rowcolor{pamerrowsoft}
\quad + \pamerplus{} & 76.7 \badgain{-0.3} & 80.9 \badgain{-0.9} & 51.6 \goodgain{+3.8} & \textbf{76.3} \goodgain{+4.2} & \textbf{71.6} \goodgain{+1.7} & 0.76M \goodgain{-71.7\%} \\

\bottomrule
\end{tabular}%
}
\vspace{-16pt}
\end{table*}

\begin{table*}[h]
\centering
\caption{
\textbf{Comparison with context management methods on WorkBuddyBench Pilot40.}
All methods use the same fixed 40 tasks. Scores are multiplied by $100$.
Parenthesized values show changes from the DeepSeek-V4-Flash base agent.
}
\label{tab:method_comparison}

\renewcommand{\arraystretch}{1.06}

\resizebox{\textwidth}{!}{%
\begin{tabular}{l|ccccc|c}
\toprule

Method & Code & Office & Sec. & Web & Avg. & Tokens $\downarrow$ \\

\midrule
\rowcolor{groupgray}
\multicolumn{7}{c}{\textbf{Base and Ours}} \\
\midrule

DeepSeek-V4-Flash~\citep{xu2026deepseek} & 72.9 & 81.6 & 44.6 & 71.0 & 67.5 & 1.89M \\

\rowcolor{pamerrowsoft}
\quad + \pamer{} & 76.3 \goodgain{+3.4} & 79.9 \badgain{-1.7} & 43.9 \badgain{-0.7} & 72.0 \goodgain{+1.0} & 68.0 \goodgain{+0.5} & 0.90M \goodgain{-52.4\%} \\

\rowcolor{pamerrowsoft}
\quad + \pamerplus{} & 71.3 \badgain{-1.6} & \textbf{86.6} \goodgain{+5.0} & 47.5 \goodgain{+2.9} & 70.0 \badgain{-1.0} & 68.9 \goodgain{+1.3} & 0.98M \goodgain{-48.1\%} \\

\midrule
\rowcolor{groupgray}
\multicolumn{7}{c}{\textbf{Simple Context Baselines}} \\
\midrule

Sliding Window ($K=5$) & 66.8 \badgain{-6.1} & 67.4 \badgain{-14.2} & 32.3 \badgain{-12.3} & 59.0 \badgain{-12.0} & 56.4 \badgain{-11.2} & 2.95M \badgain{+56.3\%} \\

Sliding Window ($K=10$) & 74.4 \goodgain{+1.4} & 73.1 \badgain{-8.6} & 40.1 \badgain{-4.5} & 61.0 \badgain{-10.0} & 62.1 \badgain{-5.4} & 1.98M \badgain{+5.1\%} \\

Sliding Window ($K=20$) & 79.3 \goodgain{+6.4} & 85.1 \goodgain{+3.5} & 40.3 \badgain{-4.3} & 70.0 \badgain{-1.0} & 68.7 \goodgain{+1.1} & 2.02M \badgain{+7.0\%} \\

Periodic Summary ($n=3$) & 65.4 \badgain{-7.5} & 67.7 \badgain{-14.0} & 26.3 \badgain{-18.3} & \textbf{82.7} \goodgain{+11.7} & 60.5 \badgain{-7.0} & 2.21M \badgain{+17.2\%} \\

Periodic Summary ($n=5$) & \textbf{83.6} \goodgain{+10.7} & 75.7 \badgain{-5.9} & 40.0 \badgain{-4.5} & 81.3 \goodgain{+10.3} & \textbf{70.2} \goodgain{+2.6} & 2.37M \badgain{+25.5\%} \\

\midrule
\rowcolor{groupgray}
\multicolumn{7}{c}{\textbf{API-based Methods}} \\
\midrule

PACE~\citep{wei2026pace}
& 70.4 \badgain{-2.6}
& 55.8 \badgain{-25.9}
& 43.5 \badgain{-1.0}
& 64.0 \badgain{-7.0}
& 58.4 \badgain{-9.1}
& 1.10M \goodgain{-41.6\%}
\\

LLMLingua-2~\citep{pan2024llmlingua}
& 64.6 \badgain{-8.3}
& 73.3 \badgain{-8.3}
& 40.1 \badgain{-4.5}
& 70.0 \badgain{-1.0}
& 62.0 \badgain{-5.5}
& 2.40M \badgain{+27.3\%}
\\

SelfCompact~\citep{li2026self}
& 76.8 \goodgain{+3.9}
& 79.6 \badgain{-2.0}
& 36.2 \badgain{-8.4}
& 71.0 \neutralgain{0.0}
& 65.9 \badgain{-1.6}
& 1.57M \goodgain{-17.0\%}
\\

ACON-Core~\citep{kang2025acon}
& 71.8 \badgain{-1.1}
& 68.6 \badgain{-13.1}
& \textbf{47.9} \goodgain{+3.4}
& 65.0 \badgain{-6.0}
& 63.3 \badgain{-4.2}
& 1.40M \goodgain{-25.7\%}
\\

Self-GC~\citep{hao2026self}
& 63.5 \badgain{-9.5}
& 71.7 \badgain{-10.0}
& 43.7 \badgain{-0.9}
& 71.0 \neutralgain{0.0}
& 62.5 \badgain{-5.1}
& 1.66M \goodgain{-12.4\%}
\\

LRE~\citep{jahan2026learning}
& 62.6 \badgain{-10.3}
& 63.4 \badgain{-18.2}
& 27.1 \badgain{-17.4}
& 68.0 \badgain{-3.0}
& 55.3 \badgain{-12.2}
& 2.29M \badgain{+21.3\%}
\\

CoMem~\citep{zhang2026comem}
& 56.4 \badgain{-16.5}
& 59.9 \badgain{-21.8}
& 2.5 \badgain{-42.1}
& 58.0 \badgain{-13.0}
& 44.2 \badgain{-23.3}
& \textbf{0.77M} \goodgain{-59.5\%}
\\

SAM~\citep{hu2026sam}
& 67.6 \badgain{-5.3}
& 81.6 \neutralgain{0.0}
& 47.2 \goodgain{+2.7}
& 67.0 \badgain{-4.0}
& 65.9 \badgain{-1.7}
& 1.35M \goodgain{-28.7\%}
\\

SWE-Pruner~\citep{wang2026swe}
& 63.5 \badgain{-9.5}
& 72.6 \badgain{-9.0}
& 40.9 \badgain{-3.7}
& 65.0 \badgain{-6.0}
& 60.5 \badgain{-7.1}
& 1.50M \goodgain{-20.5\%}
\\

Sculptor~\citep{li2026sculptor}
& 51.2 \badgain{-21.7}
& 81.6 \neutralgain{0.0}
& 35.2 \badgain{-9.3}
& 69.0 \badgain{-2.0}
& 59.3 \badgain{-8.3}
& 1.29M \goodgain{-31.5\%}
\\

\midrule
\rowcolor{groupgray}
\multicolumn{7}{c}{\textbf{Released-policy Method}} \\
\midrule

ACM~\citep{li2026acm}
& 38.7 \badgain{-34.2}
& 55.9 \badgain{-25.7}
& 6.2 \badgain{-38.4}
& 34.0 \badgain{-37.0}
& 33.7 \badgain{-33.8}
& 2.36M
\\

\bottomrule
\end{tabular}%
}
\vspace{-10pt}
\end{table*}
\subsection{Main Results}
\label{sec:system_results}

\paragraph{Full benchmark.}
Table~\ref{tab:domain_results} shows the performance and token tradeoff
on all 260 tasks. \pamer{} reduces token usage while retaining a similar
overall score. \pamerplus{} improves the reported overall score at a
slightly higher token budget than \pamer{}. The effects differ across
domains, so the aggregate result should not be read as an improvement
on every type of task.

\subsection{Baselines and Cross Model Evaluation}
\label{sec:baselines-cross-model}

\paragraph{Context-management baselines.}
Table~\ref{tab:method_comparison} compares simple controls and published
methods on Pilot40. Periodic summarization has a higher average
score in the displayed comparison, but uses more tokens. CoMem uses
fewer tokens, but has a substantially lower score. \pamer{} and
\pamerplus{} combine relatively strong task performance with a smaller
token budget than the base agent.

\paragraph{Different model backbones.}
Table~\ref{tab:cross_model} shows reduced token consumption, but model-dependent quality changes.
\pamerplus{} improves the scores of MiMo-V2.5 and Qwen3.8-Flash,
while GPT-5.6-Luna loses performance under both variants.
The results support applicability across different task-executing
backbones, while the memory controller uses the same frozen Qwen3.5-9B
feature extractor across these evaluations. Additional cross-model usage details are reported in Appendix~\ref{app:cross-model-details}.

\begin{table*}[t]
\centering
\caption{
\textbf{Cross-model evaluation on WorkBuddyBench Pilot40.}
Scores are multiplied by $100$.
All models are evaluated on the same fixed 40 tasks, with 10 tasks per domain.
Parenthesized values show changes from the corresponding base agent.
Bold indicates the best overall score and lowest token consumption
within each backbone.
}
\label{tab:cross_model}

\renewcommand{\arraystretch}{1.08}

\resizebox{\textwidth}{!}{%
\begin{tabular}{l|ccccc|c}
\toprule
Model & Code & Office & Sec. & Web & Avg. & Tokens $\downarrow$ \\
\midrule

GPT-5.6-Luna
& 44.8
& 72.1
& 55.7
& 56.0
& \textbf{57.2}
& 676K
\\

\rowcolor{pamerrowsoft}
\quad + \pamer{}
& 54.4 \goodgain{+9.6}
& 57.8 \badgain{-14.3}
& 52.2 \badgain{-3.5}
& 55.0 \badgain{-1.0}
& 54.9 \badgain{-2.3}
& \textbf{81K} \goodgain{-88.0\%}
\\

\rowcolor{pamerrowsoft}
\quad + \pamerplus{}
& 41.4 \badgain{-3.4}
& 57.7 \badgain{-14.4}
& 55.3 \badgain{-0.4}
& 66.0 \goodgain{+10.0}
& 55.1 \badgain{-2.1}
& 97K \goodgain{-85.7\%}
\\

\midrule

GLM-5.3-Flash~\citep{zeng2026glm}
& 43.2
& 85.0
& 66.1
& 61.7
& 64.0
& 762K
\\

\rowcolor{pamerrowsoft}
\quad + \pamer{}
& 44.4 \goodgain{+1.2}
& 86.5 \goodgain{+1.5}
& 70.7 \goodgain{+4.6}
& 66.7 \goodgain{+5.0}
& \textbf{67.1} \goodgain{+3.1}
& 318K \goodgain{-58.3\%}
\\

\rowcolor{pamerrowsoft}
\quad + \pamerplus{}
& 39.6 \badgain{-3.6}
& 84.4 \badgain{-0.6}
& 70.7 \goodgain{+4.6}
& 60.0 \badgain{-1.7}
& 63.7 \badgain{-0.3}
& \textbf{310K} \goodgain{-59.3\%}
\\

\midrule

HY-3
& 49.8
& 80.6
& 59.5
& 46.0
& 59.0
& 1441K
\\

\rowcolor{pamerrowsoft}
\quad + \pamer{}
& 53.6 \goodgain{+3.8}
& 64.8 \badgain{-15.8}
& 64.2 \goodgain{+4.7}
& 51.0 \goodgain{+5.0}
& 58.4 \badgain{-0.6}
& 420K \goodgain{-70.9\%}
\\

\rowcolor{pamerrowsoft}
\quad + \pamerplus{}
& 43.6 \badgain{-6.2}
& 80.0 \badgain{-0.6}
& 67.7 \goodgain{+8.2}
& 55.0 \goodgain{+9.0}
& \textbf{61.6} \goodgain{+2.6}
& \textbf{348K} \goodgain{-75.9\%}
\\

\midrule

MiMo-V2.5~\citep{mimov25}
& 32.4
& 78.4
& 45.9
& 55.0
& 52.9
& 1114K
\\

\rowcolor{pamerrowsoft}
\quad + \pamer{}
& 44.8 \goodgain{+12.4}
& 65.9 \badgain{-12.5}
& 75.0 \goodgain{+29.1}
& 61.0 \goodgain{+6.0}
& 61.7 \goodgain{+8.8}
& \textbf{276K} \goodgain{-75.2\%}
\\

\rowcolor{pamerrowsoft}
\quad + \pamerplus{}
& 59.0 \goodgain{+26.6}
& 71.0 \badgain{-7.4}
& 75.0 \goodgain{+29.1}
& 70.0 \goodgain{+15.0}
& \textbf{68.8} \goodgain{+15.8}
& 282K \goodgain{-74.7\%}
\\

\midrule

Qwen3.8-Flash~\citep{qwen3.8flashnext}
& 76.5
& 85.0
& 62.9
& 85.0
& 77.4
& 1207K
\\

\rowcolor{pamerrowsoft}
\quad + \pamer{}
& 77.6 \goodgain{+1.1}
& 74.8 \badgain{-10.2}
& 79.8 \goodgain{+16.9}
& 80.0 \badgain{-5.0}
& 78.1 \goodgain{+0.7}
& \textbf{869K} \goodgain{-28.0\%}
\\

\rowcolor{pamerrowsoft}
\quad + \pamerplus{}
& 80.1 \goodgain{+3.6}
& 83.9 \badgain{-1.1}
& 72.6 \goodgain{+9.7}
& 100.0 \goodgain{+15.0}
& \textbf{84.2} \goodgain{+6.8}
& 874K \goodgain{-27.6\%}
\\

\bottomrule
\end{tabular}%
}
\vspace{-10pt}
\end{table*}

\FloatBarrier
\subsection{Ablation and Online Behavior}
\label{sec:online-behavior}

Average token consumption does not reveal how context changes within a
long interaction. Figure~\ref{fig:system_analysis} combines the memory
operations, component results, and context trajectories. Compression
repeatedly reduces the active history, after which context grows as new
interaction occurs. Its benefit can therefore persist across later
requests that would otherwise resend the removed history.

In the reported component results, removing compression raises token
usage, whereas removing recall lowers task reward. These patterns are
consistent with complementary roles: compression controls accumulated
history, and recall restores information that has left the active
context. Appendix~\ref{app:case-study} gives a trajectory.
The component batch is separate from the Pilot40 baseline comparison. Additional figure interpretation and reproducibility details are provided in
Appendix~\ref{app:figure-interpretation}.

\FloatBarrier
\section{Conclusion}
\label{sec:conclusion}
Memory needs can be predicted from language model states before the next
action. The examined hidden representations contain information beyond
context length and interaction metadata, while compression and recall
exhibit distinct patterns across model depth. Recent context preserves much
of this predictive signal, whereas selected historical evidence recovers
information that recency alone cannot retain. These findings motivate
\pamer{} and \pamerplus{}, which combine state guided compression with
selective access to external memory. Across the reported evaluations, the
resulting systems consistently reduce context consumption, while changes in
task performance remain dependent on the model and domain. Overall, our
results connect the internal prediction of memory needs with practical
context management for long horizon language model agents.
\subsection*{AI Use Statement}
Generative AI tools were used for language editing and literature retrieval and discovery, including identifying potentially relevant papers, search keywords, and related work, as well as surveying the current state of research on relevant topics. An LLM was also used as an annotator to identify compression and recall needs in the memory-decision corpus, as described in the paper and appendix. The annotation procedure and resulting labels were reviewed by the authors. The final selection, interpretation, and presentation of the literature and experimental results were determined by the authors. All AI-assisted content was reviewed by the authors, who take full responsibility for the final content of this work.

\subsection*{Reproducibility Statement}

We provide detailed descriptions of the data construction, question-grouped
train/validation/test splits, probe architectures and evaluation protocols,
context construction, retrieval procedures, controller configuration, token
accounting, and benchmark task sets in the main paper and appendices.
Hyperparameters and implementation-specific settings used by \pamer{} and
\pamerplus{} are reported where applicable, together with the fixed Pilot40
task manifest and the definitions of the reported evaluation metrics. These details are intended to
support independent reproduction of the reported analyses and system
evaluations.
\FloatBarrier
\clearpage
\setlength{\bibsep}{.5ex plus .8ex}
\IfFileExists{iclr2027_conference.bib}{%
  \IfFileExists{iclr2027_conference.bst}{%
    \bibliographystyle{unsrtnat}
  }{%
    \bibliographystyle{unsrtnat}
  }
  \bibliography{iclr2027_conference}
}{%
  \PackageWarningNoLine{PaMER}{The existing iclr2027_conference.bib file is required to resolve citations}
}

\clearpage
\clearpage
\appendix
\setlength{\emergencystretch}{2em}
\providecommand{\pamerapptablefont}{\small}
\providecommand{\pamerapptablesetup}{%
  \pamerapptablefont\setlength{\tabcolsep}{4.5pt}%
  \renewcommand{\arraystretch}{1.10}}
\newtheorem{pmfinalproposition}{Proposition}[section]
\newenvironment{pmproof}{\par\noindent\textit{Proof.}\ }{\hfill$\square$\par\medskip}
\renewcommand{\topfraction}{0.95}
\renewcommand{\textfraction}{0.05}
\renewcommand{\floatpagefraction}{0.8}

\section{Memory-Decision Data and Probe Analysis}
\label{app:probing-details}

This appendix details the memory-decision corpus, representation probes,
compact-context analysis, online controller, and supplementary evaluations.
The mathematical results describe the stated losses and control rules.
Probing results measure predictive information, rather than identify a
causal circuit used by the language model.

\subsection{Annotated Memory Decisions and Decision-State Reconstruction}
\label{app:annotation}
\label{app:replay}

The analysis corpus contains 72,912 decision points from 2,521
trajectories covering 600 questions. These questions are separate from
the WorkBuddyBench evaluation tasks.

An LLM annotator identifies decision states where compression or recall
is needed. For trajectory $\tau$, let $\mathcal A^{\mathrm c}(\tau)$ and
$\mathcal A^{\mathrm r}(\tau)$ denote the annotated compression and
recall decision sets. At decision step $t$, with $o_t$ denoting the
context available before the next action, the labels are
\begin{equation}
 y_t^m=\mathbf1\!\left[t\in\mathcal A^m(\tau)\right],
 \qquad m\in\{\mathrm c,\mathrm r\}.
 \label{eq:pm_annotation_targets}
\end{equation}
Here $\mathrm c$ denotes compression and $\mathrm r$ denotes recall.
Compression indicates that older context should be condensed while
preserving task-relevant information. Recall indicates that the current
reasoning state needs earlier evidence that is no longer available in
the active context. Recall is evaluated only when valid external memory
exists. A trajectory may contain multiple positive decision points.

Representation extraction uses the provenance-matched OPD iter-2
Qwen3.5-9B checkpoint, which remains frozen during probe fitting. The
annotator output and the future action are used only to define the target
and are not included in the representation input.

All states and trajectory replicas with the same \texttt{query\_id}
remain in one split. Training, validation, and testing use grouped
70/15/15 partitions with seeds 0--4. The seed controls both the grouped
split and readout initialization. Within-trajectory states remain
statistically dependent.

\subsection{Probe Architecture and Training}
\label{app:probe-training}

The primary feature is the final input-token state after the output
RMSNorm of frozen Qwen3.5-9B, $\mathbf h_t\in\mathbb R^{4096}$. In online
use, this feature extractor is separate from the task-executing model;
the controller does not require hidden-state access to an API-backed
acting model. Each target has an independent readout:
\begin{equation}
 \begin{aligned}
 p_{t,\mathrm{lin}}^m&=\sigma((\mathbf w^m)^\top\mathbf h_t+b^m),\\
 p_{t,\mathrm{mlp}}^m&=\sigma\bigl((\mathbf v^m)^\top
 \operatorname{GELU}(\mathbf W^m\mathbf h_t+\mathbf b^m)+c^m\bigr).
 \end{aligned}
 \label{eq:memory-probes}
\end{equation}
The MLP hidden width is 256. Including biases, the linear and MLP probes
have 4,097 and 1,049,089 parameters per target. Only readout parameters
are trained; the language model remains frozen.

The Full Context reference used in the compact-context analysis is the
\textbf{faithful-hint MLP} with hidden width 256. It retains the recorded
context-token hint in the serialized input. Table~\ref{tab:main-probe-reference}
reports its test means. The same configuration supplies the Full Context
teacher and reference in the compact-context analysis. The online
compression heads are the separately trained Recent2 decision-distillation
students, not the Full Context teacher.

\begin{table}[htbp]
\centering
\caption{\textbf{Full Context reference for the compact-context analysis.}
Five-seed mean test metrics for the faithful-hint MLP; F1 uses per-seed
validation-selected thresholds. Values are rounded only for display.}
\label{tab:main-probe-reference}
\pamerapptablesetup
\begin{tabular}{l|rrr}
\toprule
\textbf{Target} & AUROC $\uparrow$ & AUPRC $\uparrow$ & F1 $\uparrow$ \\
\midrule
Compression & \textbf{0.829} & 0.690 & 0.639 \\
Recall & \textbf{0.780} & 0.597 & 0.558 \\
\bottomrule
\end{tabular}
\end{table}

This MLP has the highest reported F1 for both targets among the
hint/capacity settings in Table~\ref{tab:probe-robustness}; it does not
maximize every ranking metric. In particular, the faithful-hint linear
compression probe has higher AUROC. Probe training updates only the
readout. The recorded MLP defaults are AdamW, 50 epochs, learning rate
$10^{-3}$, batch size 256, and hidden width 256. The checkpoint is chosen
by validation label loss; the classification threshold is then chosen
by validation F1.

For each training split,
\begin{equation}
 \mathcal L_m=-\frac1{N_m}\sum_{i=1}^{N_m}
 [\rho_my_i^m\log p_i^m+(1-y_i^m)\log(1-p_i^m)],\qquad
 \rho_m=\frac{N_m^-}{N_m^+}.
 \label{eq:weighted-memory-bce}
\end{equation}
The threshold for a fitted probe is selected by validation F1 and held
fixed on the test set:
\begin{equation}
 \widehat\theta_{m,s}\in\operatorname*{arg\,max}_{\theta\in\Theta}
 \operatorname{F1}_{\mathrm{val},s}
 (\mathbf1\{p_t^m\ge\theta\},y_t^m).
 \label{eq:app-threshold-selection}
\end{equation}
AUROC is the primary ranking metric. AUPRC complements it for imbalanced
targets. Classification recall is distinct from the memory operation
called recall. For positive and negative scores, empirical AUROC is
\begin{equation}
 \widehat{\operatorname{AUROC}}=
 \frac1{N^+N^-}\sum_{i:y_i=1}\sum_{j:y_j=0}
 \left[\mathbf1\{z_i>z_j\}+\frac12\mathbf1\{z_i=z_j\}\right].
 \label{eq:app-auroc}
\end{equation}

\subsection{Properties of the Weighted Readout}
\label{app:weighted-loss-properties}

For $p=\sigma(z)$, the per-example loss has the stable form
\begin{equation}
 \ell_\rho(z,y)=(\rho y+1-y)\log(1+e^z)-\rho yz,
 \qquad
 \frac{\partial\ell_\rho}{\partial z}=(\rho y+1-y)p-\rho y.
 \label{eq:app-weighted-gradient}
\end{equation}
The logarithm is evaluated as
$\max(z,0)+\log(1+e^{-|z|})$ for numerical stability.

\begin{pmfinalproposition}[Population score under class weighting]
\label{prop:pm_weighted_loss}
Let $\eta(\mathbf h)=\Pr(Y=1\mid\mathbf h)$ and $\rho>0$. For
$0<\eta<1$, conditional weighted BCE has the unique minimizer
\begin{equation}
 p_\rho^*(\mathbf h)=\frac{\rho\eta}{1-\eta+\rho\eta},\qquad
 \operatorname{logit}p_\rho^*=\operatorname{logit}\eta+\log\rho.
 \label{eq:app-weighted-optimum}
\end{equation}
\end{pmfinalproposition}
\begin{pmproof}
The conditional risk is $R(p)=-\rho\eta\log p-(1-\eta)\log(1-p)$.
Solving $R'(p)=-\rho\eta/p+(1-\eta)/(1-p)=0$ gives the stated
minimizer. Since
$R''(p)=\rho\eta/p^2+(1-\eta)/(1-p)^2>0$, it is unique. Taking log
odds gives the second identity. The inverse relation is
\begin{equation}
 \eta=\frac{p_\rho^*}{\rho-(\rho-1)p_\rho^*}.
 \label{eq:app-weighted-inverse}
\end{equation}
Boundary cases follow by limits.
\end{pmproof}

The weighting changes the probability interpretation of the optimal score.
It does not certify calibration of a finite fitted probe. Validation-based
threshold selection therefore remains part of the protocol.

\subsection{Observable Controls, Probe Capacity, and Memory Hints}
\label{app:surface-controls}
\label{app:probe-robustness}

The controls measure whether length, interaction progress, tool metadata,
or surface text explain the memory-decision signal. Recent2 text controls
use a shorter input than the Full Context hidden-state probes. Their
comparison is not a matched-token-budget experiment.

\begin{table*}[!htbp]
\centering
\caption{\textbf{Memory-decision prediction controls.} Diagnostic rows show mean $\pm$ reported standard deviation over five question-grouped splits. The faithful-hint MLP reference is listed separately with mean metrics.}
\label{tab:surface-controls}
\pamerapptablesetup
\setlength{\tabcolsep}{4pt}
\begin{tabular}{l|cc|cc}
\toprule
\textbf{Signal} & \multicolumn{2}{c|}{\textbf{Compression}} & \multicolumn{2}{c}{\textbf{Recall}} \\
& AUROC $\uparrow$ & AUPRC $\uparrow$ & AUROC $\uparrow$ & AUPRC $\uparrow$ \\
\midrule
Length + turn & $0.751 \pm 0.066$ & $0.533 \pm 0.053$ & $0.552 \pm 0.074$ & $0.336 \pm 0.036$ \\
Last tool type & $0.562 \pm 0.026$ & $0.335 \pm 0.054$ & $0.527 \pm 0.056$ & $0.315 \pm 0.043$ \\
Last action type & $0.583 \pm 0.019$ & $0.346 \pm 0.057$ & $0.535 \pm 0.056$ & $0.319 \pm 0.039$ \\
Length + turn + tool type & $0.758 \pm 0.056$ & $0.543 \pm 0.054$ & $0.549 \pm 0.058$ & $0.361 \pm 0.066$ \\
\midrule
Recent2 TF-IDF & $0.616 \pm 0.027$ & $0.416 \pm 0.072$ & $0.610 \pm 0.056$ & $0.398 \pm 0.075$ \\
Recent2 text embedding & $0.546 \pm 0.015$ & $0.363 \pm 0.047$ & $0.594 \pm 0.041$ & $0.385 \pm 0.086$ \\
\midrule
\rowcolor{pamerrowsoft}
Full Context (faithful-hint MLP) & $0.829$ & $0.690$ & $0.780$ & $0.597$ \\
\bottomrule
\end{tabular}
\end{table*}

The Full Context faithful-hint MLP reference outperforms the tested observable
controls for both targets. Table~\ref{tab:probe-robustness} reports
additional probe-capacity and memory-hint measurements. The no-hint
setting removes explicit memory-token hints, whereas the faithful setting
preserves the recorded hint.

\begin{table*}[!htbp]
\centering
\caption{\textbf{Additional probe-capacity and memory-hint measurements.} Mean AUROC, AUPRC, and validation-threshold F1 over five splits. Each row denotes the named probe and hint configuration.}
\label{tab:probe-robustness}
\pamerapptablesetup
\setlength{\tabcolsep}{6pt}
\begin{tabular}{l|ccc}
\toprule
\textbf{Setting} & AUROC $\uparrow$ & AUPRC $\uparrow$ & F1 $\uparrow$ \\
\midrule
\rowcolor{groupgray}
\multicolumn{4}{c}{\textbf{Compression}} \\
No hint, linear & 0.831 & 0.670 & 0.601 \\
No hint, MLP & 0.827 & 0.665 & 0.606 \\
\midrule
Faithful hint, linear & 0.836 & 0.697 & 0.624 \\
Faithful hint, MLP & 0.829 & 0.690 & \textbf{0.639} \\
\midrule
\rowcolor{groupgray}
\multicolumn{4}{c}{\textbf{Recall}} \\
No hint, linear & 0.765 & 0.585 & 0.526 \\
No hint, MLP & 0.763 & 0.592 & 0.541 \\
\midrule
Faithful hint, linear & 0.774 & 0.598 & 0.543 \\
Faithful hint, MLP & 0.780 & 0.597 & \textbf{0.558} \\
\bottomrule
\end{tabular}
\end{table*}

\subsection{Layerwise Diagnostic Measurements}
\label{app:layerwise-definition}
\label{app:layerwise-metrics}

The layerwise study uses independently fitted no-hint linear readouts on
Qwen3.5-9B to compare depths and pooling locations. The faithful-hint MLP
used as the Full Context reference in the compact-context analysis is
reported separately and is not part of the layerwise comparison. The
final linear diagnostic is retained so that a same-capacity pooling
comparison remains available.

At layer $\ell$, independently trained probes use either the last input
state or mean pooling:
\begin{equation}
 \mathbf h_{t,\mathrm{last}}^{(\ell)}=\phi^{(\ell)}(o_t)[-1],\qquad
 \mathbf h_{t,\mathrm{mean}}^{(\ell)}=
 \frac1{n_t}\sum_{j=1}^{n_t}\phi^{(\ell)}(o_t)[j].
 \label{eq:app-pooling}
\end{equation}
\begin{table*}[!htbp]
\centering
\caption{\textbf{Layerwise linear diagnostics and the compact-context reference.}
Diagnostic rows show mean $\pm$ reported standard deviation over five
question-grouped splits. The bottom row uses the faithful-hint MLP
reference from the compact-context analysis and is not part of a
same-capacity layer comparison.}
\label{tab:layerwise-metrics}
\pamerapptablesetup
\setlength{\tabcolsep}{3.3pt}
\begin{tabular}{ll|cc|cc}
\toprule
\textbf{Representation} & \textbf{Pooling} & \multicolumn{2}{c|}{\textbf{Compression}} & \multicolumn{2}{c}{\textbf{Recall}} \\
& & AUROC $\uparrow$ & AUPRC $\uparrow$ & AUROC $\uparrow$ & AUPRC $\uparrow$ \\
\midrule
Block 4 & Last & $0.757 \pm 0.051$ & $0.563 \pm 0.063$ & $0.722 \pm 0.058$ & $0.526 \pm 0.087$ \\
 & Mean & $0.706 \pm 0.027$ & $0.483 \pm 0.063$ & $0.632 \pm 0.073$ & $0.435 \pm 0.079$ \\
\midrule
Block 8 & Last & $0.799 \pm 0.044$ & $0.623 \pm 0.082$ & $0.768 \pm 0.053$ & $0.594 \pm 0.112$ \\
 & Mean & $0.735 \pm 0.044$ & $0.493 \pm 0.068$ & $0.690 \pm 0.058$ & $0.483 \pm 0.093$ \\
\midrule
Block 16 & Last & $0.810 \pm 0.031$ & $0.644 \pm 0.082$ & $0.780 \pm 0.054$ & $0.583 \pm 0.108$ \\
 & Mean & $0.762 \pm 0.048$ & $0.547 \pm 0.064$ & $0.741 \pm 0.048$ & $0.529 \pm 0.099$ \\
\midrule
Block 24 & Last & $0.814 \pm 0.024$ & $0.649 \pm 0.089$ & $0.765 \pm 0.061$ & $0.577 \pm 0.116$ \\
 & Mean & $0.756 \pm 0.049$ & $0.534 \pm 0.047$ & $0.726 \pm 0.059$ & $0.515 \pm 0.100$ \\
\midrule
Block 32 & Last & $0.824 \pm 0.032$ & $0.659 \pm 0.087$ & $0.743 \pm 0.088$ & $0.560 \pm 0.118$ \\
 & Mean & $0.708 \pm 0.059$ & $0.478 \pm 0.069$ & $0.692 \pm 0.045$ & $0.474 \pm 0.066$ \\
\midrule
Final RMSNorm & Last & $0.831 \pm 0.036$ & $0.670 \pm 0.100$ & $0.765 \pm 0.067$ & $0.584 \pm 0.106$ \\
 & Mean & $0.712 \pm 0.058$ & $0.479 \pm 0.074$ & $0.696 \pm 0.047$ & $0.479 \pm 0.067$ \\
\midrule
\multicolumn{6}{l}{\textit{Compact-context Full Context reference: faithful-hint MLP}} \\
\rowcolor{pamerrowsoft}
Final RMSNorm & Last & $0.829$ & $0.690$ & $0.780$ & $0.597$ \\
\bottomrule
\end{tabular}
\end{table*}

At every intermediate checkpoint with both pooling measurements, the
last-token linear readout outperforms mean pooling. Recall reaches its
highest listed intermediate-layer mean at block 16, while compression
AUROC increases across the listed intermediate last-token checkpoints.
The compact-context Full Context reference is reported separately and is
not used to claim a matched readout-capacity comparison with every
diagnostic row. The final state is used for online compression; the recall probe
remains an analysis target, not a retrieval gate.

\begin{pmfinalproposition}[Logit separation is not a ranking metric]
\label{prop:pm_logit_scale}
For $z=\mathbf w^\top\mathbf h+b$ with finite class-conditional means,
\begin{equation}
 \Delta=\mathbb E[z\mid Y=1]-\mathbb E[z\mid Y=0]
 =\mathbf w^\top(\boldsymbol\mu_+-\boldsymbol\mu_-).
 \label{eq:app-gap-projection}
\end{equation}
Replacing $z$ by $z'=az+c$, $a>0$, preserves AUROC and changes the gap
to $\Delta'=a\Delta$.
\end{pmfinalproposition}
\begin{pmproof}
Linearity of expectation gives the displayed gap; the intercept cancels.
For each positive--negative pair,
$z_i'>z_j'$ if and only if $z_i>z_j$, and ties are also preserved.
Every summand in Equation~\eqref{eq:app-auroc} is consequently unchanged.
Taking the difference of conditional expectations of $az+c$ gives
$\Delta'=a\Delta$.
\end{pmproof}

Layerwise logit gaps are therefore interpreted together with ranking
performance, not as scale-free measures of semantic information. The
capacity-control and layerwise tables preserve their respective source
aggregations: for example, no-hint linear recall AUPRC is 0.584743 in the
capacity-control export and 0.584336 in the layerwise export. Their
three-decimal displays, 0.585 and 0.584, are not two roundings of one
identical underlying aggregate.

\subsection{Paired Compact-Context Comparisons}
\label{app:paired-comparisons}

For paired split-level differences $d_s=M_{A,s}-M_{B,s}$ and
$\bar d=S^{-1}\sum_s d_s$, the reported two-sided sign-flip calculation is
\begin{equation}
 p_{\mathrm{flip}}=2^{-S}\sum_{\boldsymbol\epsilon\in\{-1,1\}^S}
 \mathbf1\left\{\left|S^{-1}\sum_s\epsilon_sd_s\right|\ge|\bar d|\right\}.
 \label{eq:app-sign-flip}
\end{equation}
The inferential interpretation requires sign exchangeability. With $S=5$,
the observed sign pattern and its global negation have the same absolute
statistic, so $p_{\mathrm{flip}}\ge2/2^5=0.0625$. The five-split analysis
does not establish a two-sided $p<0.05$ result under this calculation.

\begin{table}[htbp]
\centering
\caption{\textbf{Paired compact-context differences across split seeds.} C and R denote compression and recall; positive differences favor the first method. The listed tests concern compact-context interventions and do not reuse significance results from an earlier main-reference configuration.}
\label{tab:paired-comparisons}
\pamerapptablesetup
\begin{tabular}{l|crr}
\toprule
Comparison & Target / metric & Mean difference & Sign-flip $p$\\
\midrule
Recent2 KD $-$ Recent2 & C / AUROC & +0.0058 & 0.500\\
 & R / AUROC & $-$0.0080 & 0.500\\
 & C / F1 & +0.0570 & 0.0625\\
 & R / F1 & +0.0278 & 0.3125\\
Raw Top5 $-$ Hidden Bank Recent5 & C / AUROC & +0.0627 & 0.0625\\
 & R / AUROC & +0.0389 & 0.1875\\
\bottomrule
\end{tabular}
\end{table}

\section{Compact Context and Historical Evidence}
\label{app:compact-results}

\subsection{Context Construction and Evaluation}
\label{app:compact-expanded}

Let $P_t$ be the system and task prefix, $\mathcal R_t$ the two most recent
complete interaction blocks, and $\mathcal E_t$ retrieved historical
evidence. The principal constructions are
\begin{equation}
 C_t^{\mathrm{recent}}=P_t\oplus\mathcal R_t,\qquad
 C_t^{\mathrm{recover}}=P_t\oplus\mathcal E_t\oplus\mathcal R_t.
 \label{eq:app-context-constructions}
\end{equation}
The prefix and annotation labels are held fixed when changing history.
The \texttt{Current} condition is identified as \emph{prefix only} in the
saved results. It therefore supplies a no-raw-history reference; the
paper does not treat it as proof of a causal contribution of the task
prefix. Other conditions add the specified recent blocks or historical
evidence.
Relative input cost on the evaluated states is
\begin{equation}
 B(C)=100\frac{\sum_i\operatorname{Tok}(C_i)}
 {\sum_i\operatorname{Tok}(C_i^{\mathrm{full}})}.
 \label{eq:app-context-budget}
\end{equation}
This ratio measures offline text input, not full-task token consumption
or the storage cost of a hidden-state bank. Agreement with the Full
Context teacher is a separate measurement:
\begin{equation}
 \operatorname{Agr}_m=N_m^{-1}\sum_i
 \mathbf1\{\widehat y_{i,\mathrm{candidate}}^m=
            \widehat y_{i,\mathrm{full}}^m\}.
 \label{eq:app-agreement}
\end{equation}
Agreement does not replace evaluation against the annotation labels.

\begin{table*}[!htbp]
\centering
\caption{\textbf{Compact-context and historical-evidence comparison.} C and R denote compression and recall; dots denote unreported teacher agreement.}
\label{tab:compact-expanded}
\pamerapptablesetup
\setlength{\tabcolsep}{3.3pt}
\begin{tabular}{l|r|rr|rr|rr}
\toprule
\textbf{Representation} & \shortstack{Tokens(\%)} & \multicolumn{2}{c|}{\textbf{Compression}} & \multicolumn{2}{c|}{\textbf{Recall}} & \multicolumn{2}{c}{\textbf{Agreement}} \\
& & AUROC & F1 & AUROC & F1 & Comp. & Recall \\
\midrule
Full Context & 100.0 & 0.829 & 0.639 & 0.780 & 0.558 & $\cdot$ & $\cdot$ \\
\midrule
\rowcolor{groupgray}
\multicolumn{8}{c}{\textbf{Recent context}} \\
Current & 8.8 & 0.604 & 0.451 & 0.717 & 0.532 & $\cdot$ & $\cdot$ \\
Recent1 & 17.7 & 0.746 & 0.520 & 0.725 & 0.524 & $\cdot$ & $\cdot$ \\
Recent2 & 26.0 & 0.776 & 0.541 & 0.761 & 0.520 & $\cdot$ & $\cdot$ \\
Recent4 & 41.9 & 0.795 & 0.550 & 0.755 & 0.555 & $\cdot$ & $\cdot$ \\
\midrule
\rowcolor{groupgray}
\multicolumn{8}{c}{\textbf{Scalar state and decision distillation}} \\
Recent2 + scalar state & 26.0 & 0.767 & 0.569 & 0.746 & 0.552 & $\cdot$ & $\cdot$ \\
Recent4 + scalar state & 41.9 & 0.792 & 0.579 & 0.757 & 0.586 & $\cdot$ & $\cdot$ \\
\rowcolor{pamerrowsoft}
Recent2 + decision KD & 26.0 & 0.781 & 0.598 & 0.753 & 0.548 & 0.819 & 0.840 \\
Recent4 + decision KD & 41.9 & 0.794 & 0.571 & 0.758 & 0.530 & 0.824 & 0.857 \\
\midrule
\rowcolor{groupgray}
\multicolumn{8}{c}{\textbf{Recovered raw evidence}} \\
Selective Raw Top1 & 35.9 & 0.799 & 0.569 & 0.761 & 0.542 & 0.826 & 0.817 \\
Selective Raw Top3 & 52.3 & 0.809 & 0.565 & 0.766 & 0.554 & 0.826 & 0.834 \\
\rowcolor{pamerrowsoft}
Selective Raw Top5 & 66.2 & 0.832 & 0.602 & 0.769 & 0.567 & 0.859 & 0.865 \\
\midrule
\rowcolor{groupgray}
\multicolumn{8}{c}{\textbf{Hidden and recurrent memory}} \\
Hidden Bank Uniform5 & 16.8 & 0.722 & 0.526 & 0.729 & 0.540 & 0.731 & 0.779 \\
Hidden Bank Recent5 & 16.8 & 0.769 & 0.529 & 0.730 & 0.520 & 0.785 & 0.744 \\
Hidden Bank Similarity1 & 16.8 & 0.716 & 0.514 & 0.713 & 0.507 & 0.729 & 0.721 \\
Hidden Bank Similarity3 & 16.8 & 0.722 & 0.509 & 0.703 & 0.510 & 0.729 & 0.722 \\
Hidden Bank Similarity5 & 16.8 & 0.728 & 0.555 & 0.712 & 0.493 & 0.722 & 0.747 \\
Hidden Bank Similarity8 & 16.8 & 0.725 & 0.552 & 0.703 & 0.535 & 0.758 & 0.736 \\
GRU512 & 16.8 & 0.745 & 0.499 & 0.727 & 0.513 & 0.773 & 0.757 \\
Hybrid GRU + Similarity5 & 16.8 & 0.734 & 0.517 & 0.672 & 0.508 & 0.734 & 0.684 \\
\bottomrule
\end{tabular}
\end{table*}

\begin{figure*}[t]
\centering
\includegraphics[width=0.95\textwidth]{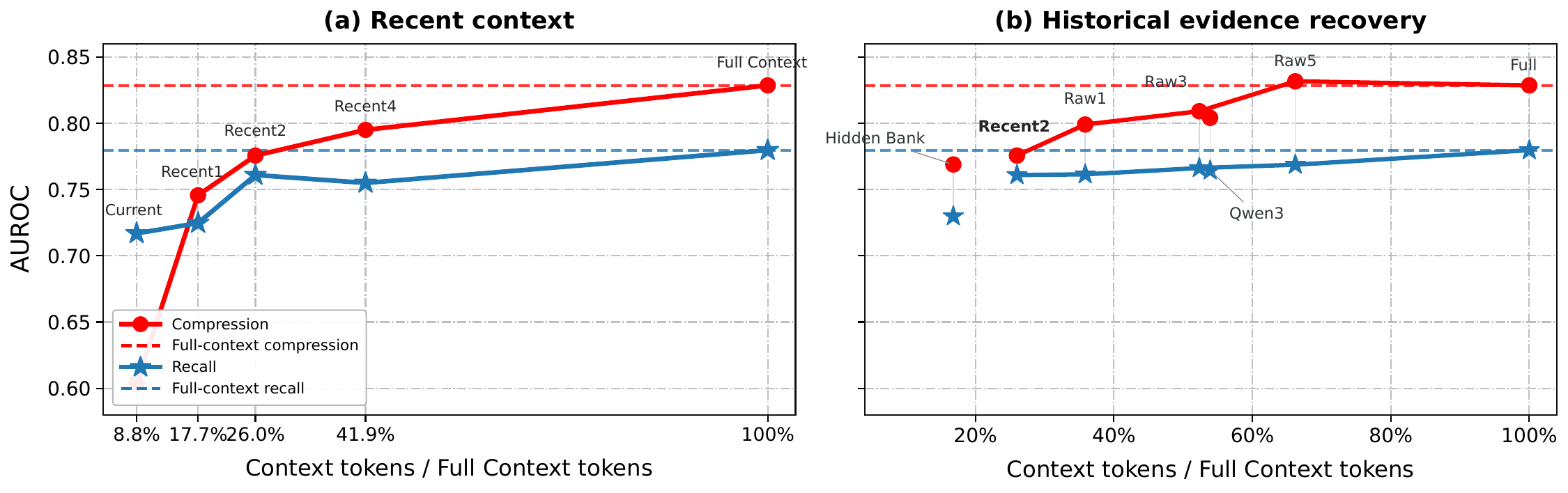}
\caption{\textbf{Decision information at different input budgets.} Left: recent context. Right: historical evidence recovery. Token fractions refer to the corresponding offline comparison.}
\label{fig:compact-tradeoff}
\end{figure*}

Recent2 uses $26.0\%$ of the full textual input and retains much of the
memory-decision signal. Additional history is useful in some, but not all,
forms: selected raw evidence outperforms the evaluated hidden-bank
alternatives, while recurrent and scalar-state additions do not
consistently improve over Recent2. Token budgets differ across these
conditions and are reported alongside predictive performance.

\subsection{Decision Distillation}
\label{app:decision-distillation}

For student logit $z$, annotation label $y$, and fixed Full Context teacher score
$q\in[0,1]$, the compression heads use
\begin{equation}
 \ell_{\mathrm{KD}}(z,y,q)=\ell_\rho(z,y)
 +\lambda\operatorname{BCE}(\sigma(z),q),\qquad\lambda=0.5.
 \label{eq:app-kd-loss}
\end{equation}
The teacher is used only during training. The logit gradient is
\begin{equation}
 \partial_z\ell_{\mathrm{KD}}=
 (\rho y+1-y+\lambda)\sigma(z)-\rho y-\lambda q.
 \label{eq:app-kd-gradient}
\end{equation}

\begin{pmfinalproposition}[Distillation as weighted soft-label learning]
\label{prop:pm_distillation}
For $\rho>0$, $\lambda\ge0$, and binary $y$, define
\begin{equation}
 a=\rho y+1-y+\lambda,\qquad
 \widetilde y=\frac{\rho y+\lambda q}{a}.
 \label{eq:app-kd-target}
\end{equation}
Then $\ell_{\mathrm{KD}}=a\operatorname{BCE}(\sigma(z),\widetilde y)$.
Moreover, conditional on the student's compact representation $h$, let
$\eta(h)=\Pr(Y=1\mid h)$ and $\mu(h)=\mathbb E[q\mid h]$. The optimal
unrestricted student probability is
\begin{equation}
 p_\lambda^*(h)=\frac{\rho\eta(h)+\lambda\mu(h)}
 {1+(\rho-1)\eta(h)+\lambda}.
 \label{eq:pm_kd_population}
\end{equation}
\end{pmfinalproposition}
\begin{pmproof}
Writing $p=\sigma(z)$, the coefficient of $-\log p$ in the combined loss
is $\rho y+\lambda q$, and that of $-\log(1-p)$ is
$1-y+\lambda(1-q)$. Their sum is $a>0$, and their ratio to $a$ gives
the soft target. After conditioning on $h$, the coefficients become
$A=\rho\eta+\lambda\mu$ and
$B=1-\eta+\lambda(1-\mu)$.
Minimizing $-A\log p-B\log(1-p)$ gives $p^*=A/(A+B)$, yielding
Equation~\eqref{eq:pm_kd_population}. When $A,B>0$, strictly positive
second derivative proves uniqueness. Endpoint solutions follow by
continuity. The conditional excess risk is
\begin{equation}
 R_h(p)-R_h(p_\lambda^*)=(A+B)\operatorname{KL}
 (\operatorname{Bern}(p_\lambda^*)\|\operatorname{Bern}(p)),
 \label{eq:pm_kd_excess}
\end{equation}
which follows directly by subtracting the two cross-entropies.
\end{pmproof}

Distillation transfers the teacher's conditional score signal into the
compact representation. It does not require the teacher to be perfectly
calibrated, nor imply that every metric improves. Empirically, compression
F1 improves more than compression AUROC, while recall AUROC decreases in
the Recent2 KD condition.

\subsection{Decision and Ranking Stability}
\label{app:decision-stability}

\begin{pmfinalproposition}[Stability under compact-context perturbations]
\label{prop:pm_compact_stability}
Apply the same fixed linear readout $z=\mathbf w^\top\mathbf h+b$ to full
and compact representations. Suppose
$\|\mathbf h_f-\mathbf h_c\|_2\le\varepsilon$ and put
$\delta=\|\mathbf w\|_2\varepsilon$. For a probability threshold $0<\theta<1$ and
$a=\operatorname{logit}\theta$, the two decisions agree whenever
\begin{equation}
 |z_f-a|>\delta.
 \label{eq:app-margin-condition}
\end{equation}
If the score bound holds almost surely under the evaluated state
distribution (including the empirical distribution on fixed examples), then
\begin{equation}
 \Pr(\widehat Y_f\ne\widehat Y_c)\le\Pr(|Z_f-a|\le\delta).
 \label{eq:app-disagreement-bound}
\end{equation}
For fixed sets containing $N^+>0$ positive and $N^->0$ negative examples,
it also implies
\begin{equation}
 \bigl|\widehat{\operatorname{AUROC}}_f-
       \widehat{\operatorname{AUROC}}_c\bigr|
 \le\frac1{N^+N^-}\sum_{i:y_i=1}\sum_{j:y_j=0}
 \mathbf1\{|z_{f,i}-z_{f,j}|\le2\delta\}.
 \label{eq:pm_auc_stability}
\end{equation}
\end{pmfinalproposition}
\begin{pmproof}
Cauchy--Schwarz gives
$|z_f-z_c|\le\|\mathbf w\|_2\|\mathbf h_f-\mathbf h_c\|_2\le\delta$.
A logit farther than $\delta$ from $a$ cannot cross the threshold, proving
the decision statements. For any positive--negative pair, the difference
between its compact and full score gaps is at most $2\delta$. A full gap
with absolute value greater than $2\delta$ therefore retains its sign.
Its AUROC summand is unchanged. Every remaining pair can change that
summand by at most one, including ties. Summing these bounds and dividing
by $N^+N^-$ proves Equation~\eqref{eq:pm_auc_stability}.
\end{pmproof}

These sufficient conditions concern the same readout and a stated
representation bound. They are not numerical guarantees for the separately
trained probes in the experimental tables.

\subsection{Retrieval Controls and Evidence Budgets}
\label{app:retrieval-bridge}
\label{app:retrieval-budgets}

The offline retrieval study compares saved raw historical blocks.
Selective Raw Top$K$ uses token-frequency cosine similarity between the
current last block and eligible older raw blocks, excluding the retained
Recent2 blocks. This lexical selector is not a learned oracle and does
not use future actions or labels to rank the evidence. Qwen conditions
instead use Qwen3-Embedding-8B and the recorded task/Recent2 query
serialization. Offline Qwen retrieval returns raw text, whereas the
canonical online \pamer{} controller retrieves stored summaries. The
two settings have different retrieval units; offline raw-block results
do not establish an online summary-retrieval effect.

\begin{table*}[!htbp]
\centering
\caption{\textbf{Offline retrieval controls.} Token percentages are relative to Full Context; agreement is measured against the corresponding teacher.}
\label{tab:retrieval-bridge}
\pamerapptablesetup
\setlength{\tabcolsep}{6pt}
\begin{tabular}{l|r|rr|rr}
\toprule
\textbf{Retriever} & \shortstack{Tokens(\%)} & \multicolumn{2}{c|}{\textbf{AUROC}} & \multicolumn{2}{c}{\textbf{Agreement}} \\
& & Comp. & Recall & Comp. & Recall \\
\midrule
Random Top3 & 47.2 & 0.793 & 0.752 & 0.786 & 0.849 \\
Recency Top3 & 46.2 & 0.792 & 0.753 & 0.818 & 0.857 \\
\midrule
Selective Raw Top1 & 35.9 & 0.799 & 0.761 & 0.826 & 0.817 \\
Selective Raw Top3 & 52.3 & 0.809 & 0.766 & 0.826 & 0.834 \\
\rowcolor{pamerrowsoft}
Selective Raw Top5 & 66.2 & 0.832 & 0.769 & 0.859 & 0.865 \\
\midrule
Qwen Top1 & 37.8 & 0.794 & 0.753 & 0.793 & 0.839 \\
\rowcolor{pamerrowsoft}
Qwen Top3 & 53.9 & 0.804 & 0.764 & 0.838 & 0.838 \\
Qwen Top5 & 67.7 & 0.816 & 0.775 & 0.834 & 0.872 \\
\bottomrule
\end{tabular}
\end{table*}

Increasing retrieval cardinality changes both evidence and input cost.
Top3 is an intermediate operating point, not a measured maximizer of
AUROC. For complete-block costs $\ell_j\ge0$, an evidence budget obeys
\begin{equation}
 \sum_{j\in\mathcal J_t}\ell_j\le B_{\mathrm{evidence}}.
 \label{eq:app-evidence-cap}
\end{equation}
The prefix, recent context, and serialization overhead are additional
parts of the final request.

\begin{table*}[!htbp]
\centering
\caption{\textbf{Qwen Top3 evidence budgets.} Mean, P95, and maximum describe recalled evidence tokens, not total prompt length.}
\label{tab:token-caps}
\pamerapptablesetup
\setlength{\tabcolsep}{6pt}
\begin{tabular}{l|rrr|rr}
\toprule
\textbf{Variant} & \multicolumn{3}{c|}{\textbf{Recalled tokens}} & \multicolumn{2}{c}{\textbf{Agreement}} \\
& Mean & P95 & Max & Comp. & Recall \\
\midrule
Top3 cap 8K & 5,540 & 8,021 & 8,192 & 0.818 & 0.835 \\
Top3 cap 16K & 8,765 & 14,801 & 16,383 & 0.843 & 0.837 \\
Top3 uncapped & 9,167 & 16,509 & 32,245 & 0.838 & 0.838 \\
\bottomrule
\end{tabular}
\end{table*}

The 16K cap retains agreement close to the uncapped condition, whereas
the 8K cap reduces agreement. These are offline evidence-budget
measurements, not a sweep of online task reward.

\section{Online Memory Controller}
\label{app:controller-contract}

\subsection{Compression Voting and Deferred Retrieval}
\label{app:pamer-algorithm}

After removing duplicate blocks already represented in stored memories,
write the working history as
$C_t^{\mathrm{eff}}=P_t\oplus\mathcal O_t\oplus\mathcal R_t$.
Five Recent2 decision-distillation MLP heads read the final normalized
Qwen3.5-9B state of the serialized $P_t\oplus\mathcal R_t$ input.
The saved probe system prompt and context-token hint are applied before
serialization. With per-seed validation thresholds $\theta_s$,
\begin{equation}
 v_{t,s}=\mathbf1\{p_{t,s}^{\mathrm c}\ge\theta_s\},\qquad
 u_t=\mathbf1\{\operatorname{Tok}(\mathcal O_t)\ge1024\}
 \mathbf1\left\{\sum_{s=1}^5v_{t,s}\ge3\right\}.
 \label{eq:app-voting}
\end{equation}
On successful compression, an entry stores summary $s_j$, its embedding,
and protocol-aware identities $\mathcal B_j$ of the covered raw blocks:
\begin{equation}
 M_j=(s_j,\mathbf e_j,\mathcal B_j),\qquad\mathbf e_j=E(s_j).
 \label{eq:app-memory-entry}
\end{equation}
The encoder is Qwen3-Embedding-8B with 1,024 output dimensions, distinct
from the 4,096-dimensional probe feature. Raw blocks leave active context
only after successful storage. The same request receives neither the new
summary nor an immediate recall; its reduced context is
$P_t\oplus\mathcal R_t$.

On a later request, the task and Recent2 define an embedding query.
For nonzero embeddings,
\begin{equation}
 a_{t,j}=\frac{\mathbf q_t^\top\mathbf e_j}
 {\|\mathbf q_t\|_2\|\mathbf e_j\|_2},\qquad
 \mathcal I_t=\operatorname{TopK}_j(a_{t,j}),\quad K=3.
 \label{eq:app-summary-retrieval}
\end{equation}
At most the available number of memories is retrieved. Query text is
serialized from the task and Recent2 and limited to its final 30,000
characters. Selected summaries are inserted after the prefix and before
\emph{all} retained raw blocks, not in place of unarchived older blocks:
\begin{equation}
 C_t^{\mathrm{recall}}=P_t\oplus\mathcal E_t\oplus\mathcal O_t\oplus\mathcal R_t.
 \label{eq:app-online-retained-history}
\end{equation}
Here $\mathcal E_t$ contains the selected summaries. The learned recall
probe does not gate this branch. A failed probe skips compression but
may still use existing memory; a successful compression request never
performs recall on that same request.

\begin{algorithm}[t]
\caption{Online context transformation with \pamer{}}
\label{alg:pamer-online}
\DontPrintSemicolon
\KwRequire{Canonical request $C_t^{\mathrm{raw}}$, persisted memory
$\mathcal M_t$, frozen Qwen3.5-9B, five Recent2 KD heads and thresholds.}
\If{the session identifier or conversation/tool structure is invalid}{
  \Return $C_t^{\mathrm{raw}}$\;
}
Load valid memory; if its saved state is corrupt, discard the invalid state and start with empty memory\;
Construct $C_t^{\mathrm{eff}}$ by removing raw blocks already covered by valid persisted memories\;
Split $C_t^{\mathrm{eff}}$ into $P_t$, $\mathcal O_t$, and $\mathcal R_t$\;
Set $u_t\leftarrow0$\;
\If{$\operatorname{Tok}(\mathcal O_t)\ge1024$}{
  Extract the final Qwen3.5-9B state and evaluate the five heads\;
  \eIf{feature extraction and head evaluation succeed}{
    Set $u_t\leftarrow\mathbf1\{\sum_s\mathbf1[p_{t,s}\ge\theta_s]\ge3\}$\;
  }{
    Record the probe error and leave $u_t=0$\tcp*[r]{Existing recall remains possible}
  }
}
\If{$u_t=1$}{
  Generate a summary and its embedding, and validate protocol-aware source identities\;
  Require successful storage and a strict reduction in active request tokens\;
  \eIf{all transaction checks succeed}{
    Commit the memory and \Return $P_t\oplus\mathcal R_t$\tcp*[r]{No same-request recall}
  }{
    Commit no partial memory and \Return $C_t^{\mathrm{raw}}$\;
  }
}
\If{$\mathcal M_t\ne\varnothing$}{
  Retrieve at most three summaries using the task and Recent2\;
  Construct $C'_t=P_t\oplus\mathcal E_t\oplus\mathcal O_t\oplus\mathcal R_t$\;
  \eIf{retrieval and final tool-protocol validation succeed}{
    \Return $C'_t$\;
  }{
    \Return $C_t^{\mathrm{raw}}$\;
  }
}
\Return $C_t^{\mathrm{eff}}$\;
\end{algorithm}

\subsection{Stability of Voting and Retrieval}

\begin{pmfinalproposition}[Majority voting without an independence assumption]
\label{prop:pm_voting}
At an eligible state with binary target $Y$, let $E_s$ indicate an
incorrect decision by head $s$. The five-head majority error satisfies
\begin{equation}
 \Pr(E_{\mathrm{maj}}\mid\mathcal C)
 \le\min\left\{1,\frac13\sum_{s=1}^5\Pr(E_s\mid\mathcal C)\right\}
 \label{eq:pm_vote_error}
\end{equation}
for any common conditioning information $\mathcal C$. Further, if at
least three heads retain their original majority vote under an input
perturbation, the majority decision is unchanged.
\end{pmfinalproposition}
\begin{pmproof}
A wrong majority requires at least three wrong heads, hence pointwise
$\mathbf1\{E_{\mathrm{maj}}\}\le\frac13\sum_s\mathbf1\{E_s\}$.
Taking conditional expectations gives the probability bound. For
stability, three unchanged votes already constitute a majority, regardless
of the remaining two votes. In particular, a head with logit $z_s$ retains
its vote if $|z_s-\operatorname{logit}\theta_s|$ exceeds its maximum logit
perturbation.
\end{pmproof}

The bound allows dependent head errors and is not a claim that the ensemble
must outperform every individual head.

\begin{pmfinalproposition}[Top-$K$ retrieval stability]
\label{prop:pm_retrieval_stability}
Let scores $a_j$ and perturbed scores $a'_j$ satisfy
$|a'_j-a_j|\le\delta$ for every memory entry. If there are more than $K$
entries and the ordered-score boundary satisfies
\begin{equation}
 a_{(K)}-a_{(K+1)}>2\delta,
 \label{eq:pm_retrieval_margin}
\end{equation}
then the retrieved Top-$K$ set is unchanged. For unit-normalized memory
embeddings and fixed entries, a query change from $\widehat{\mathbf q}$
to $\widehat{\mathbf q}'$ permits
$\delta=\|\widehat{\mathbf q}'-\widehat{\mathbf q}\|_2$.
\end{pmfinalproposition}
\begin{pmproof}
For an originally selected entry $i$ and an unselected entry $j$,
$a_i-a_j\ge a_{(K)}-a_{(K+1)}$. Therefore,
$a'_i-a'_j\ge a_i-a_j-2\delta>0$. No unselected entry can overtake any
selected entry, proving set equality. For cosine scores with unit memory
vector $\widehat{\mathbf e}_j$, Cauchy--Schwarz gives
$|a'_j-a_j|=|\langle\widehat{\mathbf q}'-\widehat{\mathbf q},
\widehat{\mathbf e}_j\rangle|\le\delta$.
\end{pmproof}

This result characterizes score-margin robustness, not the semantic
sufficiency of the retrieved memories.

\subsection{Step-Level Evidence in \pamerplus{}}
\label{app:pamerplus-algorithm}
\label{app:step-granular}

\pamerplus{} denotes the step-granular evidence-selection extension.
The supplied implementation snapshot places this component in the
LLM-router integration. The canonical embedding sidecar described above
retrieves whole summaries and does not itself import the exact-step
component. The component algorithm below must therefore not be read as
evidence that every canonical \pamer{} run executed that extension.

The router is shown the current Recent2 context and the structured
summary index, including \texttt{step\_details}; it selects summary
identifiers and the smallest requested \texttt{raw\_steps\_needed} set.
The target future action is not an input. Summary identifiers locate
candidate memories but do not authorize expansion of every covered step.
Only complete raw blocks explicitly named by the selector are supplied
to the Recall Agent. The pure step-granular component validates and
normalizes this selection; it does not itself implement a compression
gate or invoke a model. For
requested identifiers $\mathcal J_t$,
\begin{equation}
 \mathcal J_t\subseteq\bigcup_{j\in\mathcal I_t}\mathcal B_j.
 \label{eq:app-authorized-steps}
\end{equation}
Partitions are checked for gaps, overlaps, missing steps, and invalid
identifiers. Selecting a summary does not authorize restoration of every
step it covers. Cache identity uses sorted summary and step identifiers;
reuse assumes unchanged associated content and configuration. The key is
computed from sorted summary and raw-step identifiers, rather than the
router's goal wording or unrelated history updates. The recorded cache
protocol distinguishes an empty selection, an invalid packet, selector
changes, material changes, and a hit. A successful packet or a declined
recall may be reused for an unchanged exact selector.

\begin{algorithm}[t]
\caption{Exact-step selection validation in the \pamerplus{} component}
\label{alg:pamerplus-selection}
\DontPrintSemicolon
\KwRequire{Stored memories, selected summary identifiers $\mathcal I_t$,
 requested raw step identifiers $\mathcal J_t$.}
\KwEnsure{Validated recalled evidence, a valid empty selection, or failure $\bot$.}
\If{$\mathcal J_t=\varnothing$}{\Return an empty evidence set\;}
Validate the partitions and provenance links of the selected memories\;
\If{the links are invalid}{\Return $\bot$\;}
Select the complete raw blocks specified by $\mathcal J_t$\;
\If{a step is missing, duplicated, or unauthorized}{
  \Return $\bot$\;
}
Compute cache identity from the sorted summary and step identifiers\;
\If{a cached result exists for this selection}{\Return the cached result\;}
Run the recall operation using only the selected evidence\;
\If{the recall call fails or the result cites steps outside $\mathcal J_t$}{
  \Return $\bot$\;
}
Cache the validated result\;
\Return the recalled evidence\;
\end{algorithm}

The caller distinguishes a valid empty selection from a failed operation.
An empty selection injects no recalled text. Invalid or failed recall
results are not injected and do not authorize deletion of raw history.
The calling controller applies its operation-specific fallback. Final
conversation and tool-structure validation is required before a
transformed request is sent to the agent. The source snapshot includes
additional router, retry, and request-scaffold differences; its historical
system comparisons do not isolate the effect of exact-step selection.

\subsection{Storage, Evidence Selection, and Token Accounting}
\label{app:controller-properties}

\begin{pmfinalproposition}[Successful storage precedes removal]
\label{prop:pm_atomic}
Suppose a compression update removes active raw blocks only after the
summary, embedding, and valid provenance links are stored successfully,
and a failed update returns the original request. Then no failed update
removes information from the returned active request. If raw-step
validation enforces Equation~\eqref{eq:app-authorized-steps}, any raw
block supplied to the recall operation comes from an authorized memory.
\end{pmfinalproposition}
\begin{pmproof}
Before the storage condition holds, the algorithm has no branch returning
the reduced context. Failure instead returns the preserved original
request, proving the first statement. In the recall branch, valid source
links are checked before raw blocks are selected. Invalid selections are
rejected. Every surviving selected block consequently has an identifier
in the authorized union, proving the second statement.
\end{pmproof}

Structural validity and provenance do not guarantee that a summary retains
all facts or that selected evidence is sufficient for answering the task.
Under additive complete-block accounting, selecting
$\mathcal J_t\subseteq\mathcal U_t:=\bigcup_{j\in\mathcal I_t}\mathcal B_j$
reduces the raw evidence volume by
\begin{equation}
 \sum_{j\in\mathcal U_t}\ell_j-\sum_{j\in\mathcal J_t}\ell_j
 =\sum_{j\in\mathcal U_t\setminus\mathcal J_t}\ell_j\ge0.
 \label{eq:pm_step_volume}
\end{equation}
This concerns raw evidence volume; selection calls and serialization can
add their own costs.

For a fixed request sequence, let $L_j$ be an archived block's length and
$m_j$ the number of avoided occurrences in future requests. Let
$A_{\mathrm{add}}$ include additional counted memory-operation tokens and
retrieval text. Then
\begin{equation}
 \Delta T_{\mathrm{history}}=\sum_jm_jL_j-A_{\mathrm{add}}.
 \label{eq:app-amortized-saving}
\end{equation}
Each avoided occurrence contributes $L_j$, giving the first sum;
subtracting added material proves the identity. Complete online runs can
have different actions and lengths, so task totals are measured directly.
Recent2 and Top3 specify numbers of blocks, not a hard total context limit.

\subsection{Implementation Settings and Failure Handling}
\label{app:controller-config}
\label{app:failure-semantics}

Summary-source limits are in characters, while model input and output
limits are in tokens. If a source exceeds approximately 40,000 characters,
messages receive coverage from their beginning and end. Retrieved entries
use a stable identifier order, and similarity scores are not inserted into
the prompt.

\begin{table*}[!htbp]
\centering
\caption{\textbf{Controller settings.} Character budgets, token budgets, and embedding dimensions are distinct quantities.}
\label{tab:controller-params}
\pamerapptablesetup
\setlength{\tabcolsep}{5pt}
\begin{tabularx}{\linewidth}{>{\raggedright\arraybackslash}p{0.41\linewidth}|X}
\toprule
\textbf{Parameter} & \textbf{Value} \\
\midrule
\rowcolor{groupgray}
\multicolumn{2}{c}{\textbf{Compression controller}} \\
Retained complete blocks & 2 \\
Feature extractor & Frozen, provenance-matched Qwen3.5-9B \\
Feature dimension & 4,096 \\
Compression heads & 5 Recent2 decision-distillation MLPs \\
Vote rule & 3 of 5 \\
Decision thresholds & Validation best F1 threshold, selected independently for each seed \\
Probe representation & Final token representation after RMSNorm \\
MLP hidden width (compact-context reference and online heads) & 256 \\
\midrule
\rowcolor{groupgray}
\multicolumn{2}{c}{\textbf{Input and output budgets}} \\
Probe maximum input & 35,000 tokens \\
Prefix retained under truncation & 2,048 tokens \\
Output reserve used in the context hint & 16,384 tokens \\
Minimum compression candidate & 1,024 tokens \\
\midrule
\rowcolor{groupgray}
\multicolumn{2}{c}{\textbf{Summary and retrieval}} \\
Summary source budget & Approximately 40,000 characters \\
Summary model & Qwen3.5-9B (default), temperature 0, thinking disabled \\
Summary maximum output & 1,024 tokens \\
Summary timeout / attempts & 100 seconds / 8 \\
Embedding timeout / attempts & 30 seconds / 8 \\
Embedding model & Qwen3-Embedding-8B \\
Embedding dimension & 1,024 \\
Recall cardinality & Top3 \\
\bottomrule
\end{tabularx}
\end{table*}

\begin{table*}[!htbp]
\centering
\caption{\textbf{Memory-operation fallback rules.}}
\label{tab:failure-semantics}
\pamerapptablesetup
\setlength{\tabcolsep}{5pt}
\begin{tabularx}{\linewidth}{>{\raggedright\arraybackslash}p{0.30\linewidth}|X}
\toprule
\textbf{Condition} & \textbf{Required behavior} \\
\midrule
Missing session identifier & Return the original request and record the condition. \\
Unsafe tool protocol & Return the original request without applying a memory transformation. \\
Candidate below 1,024 tokens & This is not a failure: skip compression; recall from existing memory remains possible. \\
Probe failure or OOM & Skip compression and record the error; existing-memory recall may still be attempted. \\
Summary, embedding, or storage failure & Return the original request, commit no incomplete memory, and remove no raw block. \\
Recall transport or shape failure & Return the original request rather than a partially rewritten context. \\
Invalid step selection or unsafe injection & Reject the recall result; the canonical controller returns the original request on unsafe injection. The step component authorizes no raw-history deletion. \\
Valid empty step selection & Inject no recalled evidence; this is not a failed memory operation. \\
Corrupt state file & Discard invalid saved memory and continue with an empty memory store; do not fabricate prior memories. \\
\bottomrule
\end{tabularx}
\end{table*}

\section{Evaluation Protocol and System Results}
\label{app:system-details}

\subsection{Benchmark and Task Aggregation}
\label{app:benchmark-protocol}

All full-benchmark comparisons use the same fixed 260
WorkBuddyBench tasks: 80 Code, 50 Office, 60 Security, and 70 Web.
All Pilot40 comparisons use the same fixed 40 tasks, with ten tasks
from each domain. No method-specific task filtering is applied.

For task rewards $r_i$, the aggregate used in this paper is
\begin{equation}
 \overline r=N^{-1}\sum_{i=1}^{N}r_i
 =\sum_d\frac{N_d}{N}\overline r_d,
 \qquad \operatorname{Score}=100\overline r.
 \label{eq:app-reward-aggregation}
\end{equation}
Let $S_d=100\overline r_d$ denote a domain score. For complete
measurements on the stated fixed sets, the corresponding aggregates are
\begin{equation}
 \begin{aligned}
 S_{260}
 &=\frac{80S_{\mathrm{Code}}+50S_{\mathrm{Office}}
        +60S_{\mathrm{Security}}+70S_{\mathrm{Web}}}{260},\\
 S_{40}
 &=\frac{S_{\mathrm{Code}}+S_{\mathrm{Office}}
        +S_{\mathrm{Security}}+S_{\mathrm{Web}}}{4}.
 \end{aligned}
 \label{eq:app-fixed-set-averages}
\end{equation}
Because Pilot40 contains exactly ten tasks from each of the four domains,
the overall score is the equal-weight average of the four domain scores.
All averages and score changes are computed before final display rounding. Domain scores and overall scores are rounded separately
for presentation, so recomputing an overall score from the displayed
one-decimal domain values may differ slightly from the reported task-level
aggregate. The presentation retains graded rewards and token usage,
without redundant full-pass fractions in score cells.

The recorded outer agent limit is \texttt{CBC\_MAX\_TURNS}=256; this
number is a turn limit, not a task count. Tool calls, transport retries,
and internal context operations are counted separately.
ACM uses a released policy rather than the same API-backed acting model.

\subsection{Token and Context Measurements}
\label{app:system-accounting}

With counted main-agent and auxiliary calls $\mathcal C_i$ for task $i$,
\begin{equation}
 T_i=\sum_{c\in\mathcal C_i}(I_{ic}+O_{ic}),\qquad
 \overline T=N^{-1}\sum_iT_i,\qquad
 \delta_T=100\left(\frac{\overline T_{\mathrm{method}}}
 {\overline T_{\mathrm{base}}}-1\right).
 \label{eq:app-token-accounting}
\end{equation}
Cached input is already included in input tokens; reasoning output already
included in output usage is not counted again. Auxiliary operations are
counted once. Local feature extraction and embedding computation remain
additional computational work, not zero-cost operations merely because
API token counts do not cover them.

Pilot40 token usage is computed from the recorded input and output totals
over the fixed 40 tasks, with per-task usage obtained by dividing the aggregate
total by 40. Percent changes are computed from the unrounded totals before
display rounding.

Full260 token usage is aggregated over the fixed 260 tasks. The
DeepSeek-V4-Flash baseline contains 698,167,232 total tokens, corresponding
to approximately 2.69M tokens per task.

For per-turn active context $L_{it}$, the mean task peak is
\begin{equation}
 \overline P=N^{-1}\sum_i\max_t L_{it}.
 \label{eq:app-peak-context}
\end{equation}
It differs from cumulative task tokens. Monetary cost further depends on
cached-input, uncached-input, and output prices, so token reduction is not
identified with an equal percentage reduction in billed cost.

\subsection{Cross-Model Scores on Fixed Pilot40}
\label{app:cross-model-details}

All cross-model comparisons use the same fixed Pilot40 task set.
GPT-5.6-Luna, GLM-5.3-Flash, HY-3, MiMo-V2.5, and Qwen3.8-Flash are
each evaluated with the base agent, \pamer{}, and \pamerplus{} on all
40 tasks, with ten tasks per domain. No model- or method-specific task
filtering is applied. Overall scores are computed from the underlying
task-level rewards before display rounding.

\begin{table*}[!htbp]
\centering
\caption{\textbf{Cross-model results on the fixed Pilot40 task set.}
Scores are multiplied by 100. All variants use the same 40 tasks,
with ten tasks per domain. Avg. is the equal-weight average of the
four domain scores; $\Delta$ is the change from the corresponding
base agent.}
\label{tab:app-cross-scores}

\small
\setlength{\tabcolsep}{3.2pt}
\renewcommand{\arraystretch}{1.05}

\begin{tabular}{ll|rrrr|rr}
\toprule
\textbf{Backbone}
& \textbf{Variant}
& \textbf{Code}
& \textbf{Office}
& \textbf{Sec.}
& \textbf{Web}
& \textbf{Avg.}
& $\boldsymbol{\Delta}$ \\
\midrule

GPT-5.6-Luna
& Base agent & 44.8 & 72.1 & 55.7 & 56.0 & 57.2 & 0.0 \\
\rowcolor{pamerrowsoft}
& \pamer{} & 54.4 & 57.8 & 52.2 & 55.0 & 54.9 & -2.3 \\
\rowcolor{pamerrowsoft}
& \pamerplus{} & 41.4 & 57.7 & 55.3 & 66.0 & 55.1 & -2.1 \\

\midrule

GLM-5.3-Flash
& Base agent & 43.2 & 85.0 & 66.1 & 61.7 & 64.0 & 0.0 \\
\rowcolor{pamerrowsoft}
& \pamer{} & 44.4 & 86.5 & 70.7 & 66.7 & 67.1 & +3.1 \\
\rowcolor{pamerrowsoft}
& \pamerplus{} & 39.6 & 84.4 & 70.7 & 60.0 & 63.7 & -0.3 \\

\midrule

HY-3
& Base agent & 49.8 & 80.6 & 59.5 & 46.0 & 59.0 & 0.0 \\
\rowcolor{pamerrowsoft}
& \pamer{} & 53.6 & 64.8 & 64.2 & 51.0 & 58.4 & -0.6 \\
\rowcolor{pamerrowsoft}
& \pamerplus{} & 43.6 & 80.0 & 67.7 & 55.0 & 61.6 & +2.6 \\

\midrule

MiMo-V2.5
& Base agent & 32.4 & 78.4 & 45.9 & 55.0 & 52.9 & 0.0 \\
\rowcolor{pamerrowsoft}
& \pamer{} & 44.8 & 65.9 & 75.0 & 61.0 & 61.7 & +8.8 \\
\rowcolor{pamerrowsoft}
& \pamerplus{} & 59.0 & 71.0 & 75.0 & 70.0 & 68.8 & +15.8 \\

\midrule

Qwen3.8-Flash
& Base agent & 76.5 & 85.0 & 62.9 & 85.0 & 77.4 & 0.0 \\
\rowcolor{pamerrowsoft}
& \pamer{} & 77.6 & 74.8 & 79.8 & 80.0 & 78.1 & +0.7 \\
\rowcolor{pamerrowsoft}
& \pamerplus{} & 80.1 & 83.9 & 72.6 & 100.0 & 84.2 & +6.8 \\

\bottomrule
\end{tabular}
\end{table*}

\begin{table*}[!htbp]
\centering
\caption{\textbf{Cross-model token usage on fixed Pilot40.}
K denotes thousands of tokens. Totals include the counted main-agent
and auxiliary memory calls. Changes are relative to the corresponding
base agent.}
\label{tab:app-cross-usage}
\small
\setlength{\tabcolsep}{3.2pt}
\renewcommand{\arraystretch}{1.05}
\begin{tabular}{ll|rr}
\toprule
\textbf{Backbone} & \textbf{Variant}
& \shortstack{Tokens/task (K) $\downarrow$}
& \shortstack{Change (\%)} \\
\midrule
GPT-5.6-Luna & Base agent & 676 & 0.0 \\
\rowcolor{pamerrowsoft}
 & \pamer{} & 81 & -88.0 \\
\rowcolor{pamerrowsoft}
 & \pamerplus{} & 97 & -85.7 \\
\midrule
GLM-5.3-Flash & Base agent & 762 & 0.0 \\
\rowcolor{pamerrowsoft}
 & \pamer{} & 318 & -58.3 \\
\rowcolor{pamerrowsoft}
 & \pamerplus{} & 310 & -59.3 \\
\midrule
HY-3 & Base agent & 1441 & 0.0 \\
\rowcolor{pamerrowsoft}
 & \pamer{} & 420 & -70.9 \\
\rowcolor{pamerrowsoft}
 & \pamerplus{} & 348 & -75.9 \\
\midrule
MiMo-V2.5 & Base agent & 1114 & 0.0 \\
\rowcolor{pamerrowsoft}
 & \pamer{} & 276 & -75.2 \\
\rowcolor{pamerrowsoft}
 & \pamerplus{} & 282 & -74.7 \\
\midrule
Qwen3.8-Flash & Base agent & 1207 & 0.0 \\
\rowcolor{pamerrowsoft}
 & \pamer{} & 869 & -28.0 \\
\rowcolor{pamerrowsoft}
 & \pamerplus{} & 874 & -27.6 \\
\bottomrule
\end{tabular}
\end{table*}

The controller uses the same frozen Qwen3.5-9B feature extractor when the
task-executing model changes. Cross-model deployment therefore does not
require transferring a probe between different acting-model hidden spaces.

\subsection{Component Ablation}
\label{app:component-ablation}

The component ablation is evaluated in a separate execution batch and is
reported at the aggregate level.

\begin{table*}[t]
\centering
\caption{\textbf{Component ablation in a separate execution batch.}
Reward uses the original scale.}
\label{tab:app-component-results}
\pamerapptablesetup
\setlength{\tabcolsep}{6pt}
\begin{tabular}{l|rr|r}
\toprule
\textbf{Variant} & Reward $\uparrow$ & Tokens/task $\downarrow$ & Reported peak $\downarrow$ \\
\midrule
Base agent & 0.6868 & 3.066M & 280.1K \\
\midrule
\rowcolor{pamerrowsoft}
\pamer{} & 0.7000 & 0.899M & 97.5K \\
\midrule
w/o Recall & 0.6370 & 0.924M & 107.1K \\
w/o Compression & 0.6887 & 3.180M & 253.3K \\
\bottomrule
\end{tabular}
\end{table*}

Removing recall lowers reward from 0.7000 to 0.6370, while removing
compression increases tokens per task from 0.899M to 3.180M. These
patterns are consistent with the complementary roles of context reduction
and historical access, without establishing a universal causal effect.

\section{Trajectory Analysis and Evaluation Tasks}
\label{app:figure-interpretation}

\subsection{Context Dynamics}

The full-benchmark context-growth comparison uses the same fixed 260-task
list for all displayed methods. Let $L_{i,m,t}$ be active input tokens
for task $i$, method $m$, and turn $t$. The per-turn mean is
\begin{equation}
 \overline L_{m,t}=\frac1{|\mathcal I_{m,t}|}
 \sum_{i\in\mathcal I_{m,t}}L_{i,m,t},\qquad
 \mathcal I_{m,t}=\{i\in\mathcal Q_{260}: L_{i,m,t}\text{ is observed at turn }t\}.
 \label{eq:app-turnwise-context}
\end{equation}
A task that has already ended contributes no later-turn context value,
so the active trajectory count can decrease with turn even though the
underlying evaluation set is the same fixed 260 tasks for every method.
The 128K line is an analysis reference, not a controller-enforced limit.

\subsection{Compression and Recall in One Trajectory}
\label{app:case-study}
\label{sec:recall_case}

Question 558 is a multi-document identification task combining a museum
founder, a magazine editor, and biographical clues. The final answer,
Tom Rice, matches the task's recorded answer. The trajectory contains 94
turns, 81 searches, 16 document retrievals, four compression calls, and
two memory queries. Selected events show how old evidence becomes
accessible again after leaving active context. This single trace is a
qualitative illustration of memory operations, not an additional
WorkBuddyBench evaluation cohort or a change to the fixed task sets.

\begin{table*}[t]
\centering
\caption{\textbf{Selected events in the illustrative trajectory.}}
\label{tab:app-case-events}
\pamerapptablesetup
\setlength{\tabcolsep}{5pt}
\begin{tabularx}{\linewidth}{r|l|X}
\toprule
\textbf{Turn} & \textbf{Recorded operation} & \textbf{Evidence in the saved trace} \\
\midrule
19 & \texttt{manage\_context} & 41 messages are summarized into memory 1. \\
28 & \texttt{manage\_context} & 19 messages are summarized into memory 2. \\
36 & \texttt{query\_memory} & The agent requests memory 1 during continued search. \\
47 & \texttt{manage\_context} & 39 messages are summarized into memory 3. \\
71 & \texttt{manage\_context} & 49 messages are summarized into memory 4. \\
87 & \texttt{get\_document} & Document 20026 is retrieved as candidate supporting evidence. \\
93 & Final answer & The agent returns Tom Rice, matching the recorded gold answer. \\
\bottomrule
\end{tabularx}
\end{table*}

\subsection{Fixed Pilot40 Task Identifiers}
\label{app:pilot40-manifest}

The following identifiers define the fixed Pilot40 task set, with ten
tasks per domain. The same 40-task manifest is used for the Pilot40
method comparison and the cross-model evaluation. The full-benchmark
evaluation uses the fixed 260-task WorkBuddyBench list.

\begin{table*}[t]
\centering
\caption{\textbf{Fixed Pilot40 task identifiers.}}
\label{tab:app-pilot40-tasks}
\pamerapptablesetup
\renewcommand{\arraystretch}{1.02}
\setlength{\tabcolsep}{5pt}
\begin{tabularx}{\linewidth}{ll|X}
\toprule
\textbf{Domain} & \textbf{Index} & \textbf{Task identifier} \\
\midrule
Code & 1 & \path|api_contract-hard-markup_errors| \\
 & 2 & \path|api_contract-hard-openapi_params| \\
 & 3 & \path|api_contract-hard-token_errors| \\
 & 4 & \path|api_contract-hard-validation_errors| \\
 & 5 & \path|bug_fix-easy-a_crash_in_local| \\
 & 6 & \path|bug_fix-easy-filtered_relation_queryset_arg| \\
 & 7 & \path|bug_fix-easy-invalid_filterwarnings_regex_error| \\
 & 8 & \path|bug_fix-medium-error_key_uses_data_key| \\
 & 9 & \path|bug_fix-medium-errors_from_earlier_indices| \\
 & 10 & \path|bug_fix-medium-incorrect_linenos_on_fstring| \\
\midrule
Office & 11 & \path|analyst-forecast-extract-L3-018| \\
 & 12 & \path|api-usage-explain-cli-l3-001| \\
 & 13 & \path|board-material-update-timeline-excel| \\
 & 14 & \path|calendar-dida-sync-state| \\
 & 15 & \path|channel-period-compare-L4-017| \\
 & 16 & \path|cloudagent-sdk-doc-validation-report| \\
 & 17 & \path|contract-extract-L3-014| \\
 & 18 & \path|cross-week-dashboard-migration| \\
 & 19 & \path|crypto-backtest-chain-L4-002| \\
 & 20 & \path|daily-creation-checkpoint-recovery| \\
\midrule
Security & 21 & \path|agent-to-agent-injection-hard-multistep| \\
 & 22 & \path|apt-multi-source-correlation-hard-multistep| \\
 & 23 & \path|bb-bin-dns-parse-010| \\
 & 24 & \path|bb-bin-firmware-audit-007| \\
 & 25 & \path|bb-bin-format-log-004| \\
 & 26 & \path|bb-bin-int-length-005| \\
 & 27 & \path|bb-bin-ipc-cache-001| \\
 & 28 & \path|bb-bin-media-parse-008| \\
 & 29 & \path|bb-bin-oob-read-003| \\
 & 30 & \path|bb-bin-parse-crash-006| \\
\midrule
Web & 31 & \path|animated-explainer-L3-028| \\
 & 32 & \path|atmosphere-game-L4-035| \\
 & 33 & \path|blog-editor-draft-recovery-L4-059| \\
 & 34 & \path|browser-clipper-extension-L4-005| \\
 & 35 & \path|canvas-webgl-scene-L4-026| \\
 & 36 & \path|chart-generation-L2-025| \\
 & 37 & \path|checkout-incident-analysis-L4-049| \\
 & 38 & \path|city-article-theme-variants-L4-066| \\
 & 39 & \path|claims-drawer-state-review-report-L3-071| \\
 & 40 & \path|cohort-retention-dashboard-L4-054| \\
\bottomrule
\end{tabularx}
\end{table*}

\FloatBarrier
\end{document}

%% file: commands.tex
\usepackage[left=2.5cm,
right=2.5cm,
top=2.3cm,
bottom=2.3cm,
headheight=20pt,
headsep=10pt,
footskip=25pt,
letterpaper]{geometry}
\usepackage[utf8]{inputenc}
\usepackage[T1]{fontenc}
\usepackage[english]{babel}
\usepackage{amsmath,amsfonts,amssymb,amsthm,thmtools}
\usepackage{graphicx}
\usepackage{hyperref}
\usepackage{fancyhdr}
\usepackage[normalem]{ulem}
\usepackage{graphicx}
\usepackage{stfloats}
\usepackage{wrapfig}
\usepackage{epstopdf}
\usepackage{cleveref}
\usepackage{subfloat}
\usepackage{subcaption}
\usepackage{xspace}
\usepackage{enumitem}
\usepackage{listings}
\usepackage{titlesec}
\usepackage{etoolbox}
\usepackage{setspace}
\usepackage{changepage}
\usepackage{etoolbox}
\usepackage{multirow}
\usepackage{booktabs}
\usepackage{tabularx}
\usepackage{wrapfig}
\usepackage{svg}
\usepackage[percent]{overpic}
\usepackage[round]{natbib}
\usepackage[colorinlistoftodos, shadow,color=blue!30!white
]{todonotes}
\usepackage{xpatch}
\usepackage{siunitx}

\fancypagestyle{first}{\fancyfoot[R]{\small\thepage}}

\setlist[itemize]{leftmargin=1em,itemsep=0ex,topsep=0ex}
\titlespacing*{\paragraph}{0pt}{0ex plus .1ex}{1ex}
\titlespacing*{\section}{0ex}{2.3ex plus .3ex minus .0ex}{.6ex plus .3ex minus .2ex}
\titlespacing*{\subsection}{0ex}{1.5ex plus .3ex minus .5ex}{.4ex plus .2ex minus .1ex}
\titlespacing*{\subsubsection}{0ex}{1.2ex plus .3ex minus .3ex}{.3ex plus .2ex minus .2ex}

\xapptocmd\normalsize{%
\abovedisplayskip=.8em plus .2em minus .2em
\belowdisplayskip=.6em plus .1em minus .1em
\abovedisplayshortskip=.8em plus .2em minus .2em
\belowdisplayshortskip=.6em plus .1em minus .1em
}{}{}

\setcitestyle{numbers}
\renewcommand{\cite}[1]{\citep{#1}}

\definecolor{mydarkblue}{rgb}{0.0,0.15,0.7}
\hypersetup{%
colorlinks=true,
linkcolor=mydarkblue,
citecolor=mydarkblue,
filecolor=mydarkblue,
urlcolor=mydarkblue}

\makeatletter

  \renewcommand{\maketitle}{%
    \begingroup
      {\centering\LARGE\@title\par}%
      \vskip 1em
      \centering
      \begin{tabular}[t]{@{}c@{}}\strut\@author\strut\end{tabular}%
      \vskip 0.3in minus 0.1in
    \endgroup
  }
\makeatother


%% file: authors.tex
\author{%
  \mbox{Mingxuan Wang\textsuperscript{1}}\quad \mbox{Guorun Yao\textsuperscript{1}}\quad \mbox{Fei Luo\textsuperscript{1}}\quad \mbox{Yinglong Guo\textsuperscript{1}}\quad \mbox{Chao Ning\textsuperscript{1}}\\[2pt]
  \mbox{Bo Wang\textsuperscript{1}}\quad \mbox{Hongyue Chen\textsuperscript{1}}\quad \mbox{Yanbiao Ma\textsuperscript{2,*}}\quad \mbox{Jungong Han\textsuperscript{3,*}}\\[4pt]
  \textsuperscript{1}TierFlow Team\\
  \textsuperscript{2}Gaoling School of Artificial Intelligence, Renmin University of China\\
  \textsuperscript{3}Tsinghua University\\[3pt]
  \textsuperscript{*}Corresponding authors.\quad \href{mailto:ybma1998@ruc.edu.cn}{\texttt{ybma1998@ruc.edu.cn}}%
}
\hypersetup{pdfauthor={Mingxuan Wang, Guorun Yao, Fei Luo, Yinglong Guo, Chao Ning, Bo Wang, Hongyue Chen, Yanbiao Ma, Jungong Han}}